%% file: paper.tex
\documentclass[final,12pt]{preprint} %

\usepackage{times}

\usepackage{svg}
\usepackage{bm}
\usepackage{pict2e}
\usepackage{amsfonts}
\usepackage[normalem]{ulem}
\usepackage{multicol}
\usepackage{caption}
\usepackage{subcaption}
\usepackage{mathtools}
\usepackage{subfiles}
\usepackage{diagbox}
\usepackage{float}
\usepackage{stmaryrd}
\usepackage{calc}
\usepackage{booktabs}
\usepackage{xspace}
\usepackage{algpseudocode}
\usepackage{enumitem}
\usepackage{notation}
\usepackage{wrapfig}
\usepackage{makecell}
\usepackage[
raggedleft
]{sidecap}   
\sidecaptionvpos{figure}{t}
\title
[Algorithmic causal structure emerging through compression]
{Algorithmic causal structure emerging through compression}
\clearauthor{\Name{Liang Wendong} \Email{wendong.liang@tue.mpg.de}\\
\addr Max Planck Institute for Intelligent Systems, T\"ubingen, Germany\\
 \Name{Simon Buchholz} \Email{simon.buchholz@tue.mpg.de}\\
 \addr Max Planck Institute for Intelligent Systems, T\"ubingen, Germany\\
 \Name{Bernhard Sch\"olkopf} \Email{bs@tue.mpg.de}\\
 \addr{Max Planck Institute for Intelligent Systems \& ELLIS Institute, Tübingen, Germany}}

\newif\ifcomments

\commentsfalse   %

\newcommand{\acomment}[1]

\begin{document}

\maketitle

\begin{abstract}%
  We explore the relationship between causality, symmetry, and compression. We build on and generalize the known connection between learning and compression
  to a setting where causal models are not identifiable. We propose a framework where causality emerges as a  consequence of compressing data across multiple environments. We define algorithmic causality as an alternative definition of causality when traditional assumptions for causal identifiability do not hold. We demonstrate how algorithmic causal and symmetric structures can emerge from minimizing upper bounds on Kolmogorov complexity, without knowledge of intervention targets. We hypothesize that these insights may also provide a novel perspective on the emergence of causality in machine learning models, such as large language models, where causal relationships may not be explicitly identifiable.
\end{abstract}


\newpage
\tableofcontents

\newpage
\section{Introduction}
The case has been made that learning and compression are intimately related~\citep[Sec. 4.6]{blumer1987occam, mackay2003information,grunwald2007minimum,vapnik1999nature}: both are made possible by regularities in datasets. In the present paper, we seek to extend this connection beyond predictive machine learning. 
Such settings are often studied using causal models such as structural causal models (SCMs) and causal Bayesian networks (CBNs)~\citep{pearl2009causality}.
We are particularly interested in settings where data is non i.i.d.\ and originates from multiple environments which share (most of the) mechanisms or causal conditionals.
This assumption is termed sparse mechanism shift (SMS)~\citep{scholkopf2021toward}. Intuitively, SMS helps us when jointly compressing models learned across multiple environments, since shared mechanisms then need only be compressed once. However, it is not trivial to integrate this into the known compression framework. This serves as a motivation for our paper, with the goal of working towards a rigorous framework of compression in which causality emerges as a regularity bias.

As a starting point, we observe that most of the previous literature on identifiability in causal discovery and causal representation learning can be rephrased in a classical compression framework: There is (up to tolerable ambiguities) a unique model under which the data has maximal likelihood, equivalently, a unique model whose distribution has minimal cross-entropy with the data distribution. A common criticism of causality research is that identifiability generally requires strong assumptions. We are interested in cases when those assumptions do not hold, and investigate what we can still say about causality. As a motivation we remark that if the only well-defined notion of causality were to build upon identifiability subject to unrealistic assumptions, why could humans and animals possess reliable causal knowledge of aspects of the world, and how should we formalize this knowledge despite the issues of non-identifiability?

\paragraph{Contributions}
\begin{itemize}
    \item We discuss the relationship between identifiability (in causal discovery and causal representation learning) and compression (\cref{sec:prelim_ident}).
    \item To the best of our knowledge, our work is the first rigorous treatment of principled decisions (rather than only non-identifiability)
    of causal arrows with no constraints on distribution classes and no knowledge of intervention types or targets in a non-Bayesian perspective, with a weaker definition of causality (\cref{sec:algo_causality}, \cref{def:algo_causality_informal}, \cref{def:algorithmic_causality_CFMP}).
    \item Under the settings where minimum cross-entropy cannot identify causal arrows, we use a more general notion of compression, i.e., Kolmogorov complexity, to carry out model selection over the algorithmic causal models. (\cref{sec:learn_compression}) We provide computable upper bounds on Kolmogorov complexity (i.e., finite codebook complexity (\cref{def:FC_complexity})) under certain non-universal Turing machines, i.e., universal finite codebook computers (UFCC, \cref{def:UFCC}).
    \item We prove that under some UFCCs, models using causal factorizations and models with the sparsest mechanism shifts are preferred by minimizing the finite codebook complexity. (\cref{sec:case_studies}, \cref{prop:SMS}) This means that algorithmic causality emerges as a by-product when minimizing an upper bound of Kolmogorov complexity.
    We further show that some UFCCs align simplicity (i.e., short coding length) with symmetries such as invariance or equivariance under group actions (\cref{prop:inv_shorter}).
\end{itemize}

\section{Identifiability in causality and its relation with compression}\label{sec:prelim_ident}

In this section, we review the general relation of compression
and identifiability in causality, and the limitations of both two notions. We focus on CBNs since \cite{pearl2009causality} shows that identifying CBNs is strictly easier than identifying SCMs, and hence all difficulties regarding identifiability in CBNs also exist in SCMs.
An \textbf{observational CBN model} is a tuple $(G, \PP)$ where $G=(V,E)$ is a directed acyclic graph (DAG), and the distribution $\PP$ is Markov relative to $G$, i.e., $\PP(X) = \prod_i \PP(X_i | X_{\pa^G(i)})$. $\Dcal_n$ denotes n iid samples from $\PP$.
A \textbf{multi-env CBN model} is tuples $M = (G^i, \PP^i)_{i\in [I]}$ where each tuple is called an \textbf{environment (env)}, and each $\PP^i$ is Markov relative to $G^i$.
$\Dcal_n^i$ denotes n iid samples from $\PP^i$. Denote $\Mcal$ as a (multi-env) \textbf{CBN model class} that contains different (multi-env) CBN models. 

\begin{definition}[Identifiability in causal discovery]\label{def:identCDCRL}
    Given an observational CBN model class $\Mcal$ in which all the joint distributions are absolutely continuous w.r.t.\  a  measure $\mu$, we say $\Mcal$ is \emdef{identifiable} if for any $(G, \PP_\theta), (G', \PP_{\theta'}) \in \Mcal$ with $\PP_\theta(x) = \PP_{\theta'}(x)$ $\mu$-almost everywhere, we have $G=G'$.
    Given an multi-env CBN model class $\Mcal$, we say $\Mcal$ is \emdef{identifiable} if for any $(G^i, \PP_\theta^i)_{i\in [I]}, (G'^i, \PP_{\theta'}^i)_{i\in [I]} \in \Mcal$ with $\PP^i_\theta(x) = \PP^i_{\theta'}(x)$ $\mu$-almost everywhere for all $i\in [I]$, we have $G^i=G'^i$ for all $i\in [I]$.
\end{definition}

For both observational and multi-env models, we have the following well-known result, similar to that in classical statistics~\citep[e.g.][Thm 17.3]{greene2003econometric}. Defining the likelihood function $L(\theta|\Dcal_n)=p_\theta(\Dcal_n)$, we have:
\begin{restatable}{lemma}{lemIdentminCE}\label{lem:IdentminCE}\textbf{(Identifiability implies uniqueness of solution of minimum cross-entropy)\footnote{For readability we stay in the unconfounded setting and the strong version of identifiability. We can readily generalize~\cref{def:identCDCRL} and this lemma to identifiability up to an equivalence class, or generalize to the setting of causal representation learning, by changing the observed distribution to be a marginal distribution $\PP_X$ of the model distribution $\PP_{XZ}$, while forcing the invariance of $\PP_{X|Z}$.}}

    Given a CBN model class $\Mcal$, if $\Mcal$ is identifiable and its distribution class is parametric, then 
    
    (1) For the observational CBN model, the solution of maximum likelihood\\ $\argmax_{M\in \Mcal} \lim_{n\to\infty} \frac{1}{n}\log L(M|\Dcal_n)$
    is unique.
    Equivalently, the minimizer of cross-entropy
    $\argmin_{(G,\theta)\in \Mcal}\EE_{\PP_\theta^*}[-\log\PP_\theta(X)]$
    is unique.
    
    (2) For the multi-env CBN model with uniform prior over environments, the solution of maximum likelihood $\argmax_{M\in \Mcal} \lim_{n\to\infty} \sum_{i=1}^K \frac{1}{n}\log L(M|\Dcal_n^i)$
    is unique. Equivalently,
    the minimizer of the sum of cross-entropies across multi-env
    \begin{align}
    \argmin_{\left((G_i)_{i\in [I]},(\theta_i)_{i\in [I]}\right)\in \Mcal} \sum_{i=1}^I \EE_{\PP_{\theta^*_i}}[-\log\PP_{\theta_i}(X)]
    \end{align}
    is unique.
\end{restatable}
The proof is in \cref{sec:proof_sec_prelim_ident}.

By Shannon's source coding theorem (\cref{Thm:Shannon_source_coding})
the entropy, (equivalently the minimum cross-entropy) is the shortest (most compressed) average coding length for an iid sequence. 
Therefore the desideratum of identifiability research is to justify that compression (minimum cross-entropy) is the correct model selection method in causal discovery.

\paragraph{Limitations}
\begin{itemize}
    \item \textbf{Identifiability is deterministic model selection.}
    By \cref{lem:IdentminCE}, once we have identifiability results in a model class $\Mcal$, we can deterministically select the ground truth model in $\Mcal$ using maximum likelihood or minimum cross-entropy.
    It is known that without constraints of distribution class and without knowledge of the intervention targets, there is no identifiability beyond the Markov equivalence class. 
    There is extensive research on finding model classes that make causal models identifiable. 
    These correspond to hard priors, restricting the model class to a lower dimensional submanifold, resulting in subjective model pre-selection. However, the model selection problem is inescapable
    ---one can either pre-select the model class with constraints and then derive a deterministic model selection result, or 
    perform model selection directly in an unconstrained model class. The success of modern empirical machine learning is based on the latter, while identifiability, based on the former, 
    excludes parts of the model space a priori.
    \item \textbf{Intervention types and targets.}
    Many papers show identifiability results under no distribution constraints on causal mechanisms but with the knowledge of intervention types or targets. In~\cref{def:identCDCRL}, we can see that intervention types or targets are constraints in $\Gcal^I$, the $I$-th cartesian product of all graphs of $d$ nodes. For example, the assumption of ``all interventions are soft''~\citep{perry2022causal,wildberger2023interventional} implies that the multi-env graphs $(G_i)_{i\in [I]}$ have to be the same graph $G$ across all environments. 
   Under faithfulness assumption, $\Mcal=\bigsqcup_{(G_i)_{i \in [I]}\in\Gcal^I} \Theta_{(G_i)_{i \in [I]}}$, a disjoint union of model classes, each of which contains $I$-env systems that are Markov and faithful to $I$-many graphs. 
   By reducing the possible graphs, they in fact force hard priors over the probabilistic model class on multi-env systems.
   The more assumptions we have on intervention types or targets, the smaller the distribution class is, and the more chance of model misspecification there is.
   Without intervention types or targets, multi-env data are in fact \textit{correlational} data, because every new environment can arise from interventions in each variable respectively, which can be completely unrelated to the mechanisms in the observational environment. In such cases, no identifiability beyond Markov equivalence class is possible.
    \item \textbf{Entropy is not all the bits needed to encode datasets.} The length of the codebook is missing in cross-entropy, which is why identifiability theories cannot distinguish different computational models that compute the same probabilistic model. We discuss this in detail in~\cref{sec:comment_entropy}.
\end{itemize}

\paragraph{The central questions in this paper:} 
\begin{enumerate}
    \item If we observe multi-env data, but have no certain knowledge about intervention types or targets, no hard prior over the model class, and thus no identifiability guarantees, can we decide on the causal relationship between variables, with a weaker definition of causality?
    \item How to learn such a causal relationship? What is its relation with compression?
\end{enumerate}

\section{Algorithmic causality}\label{sec:algo_causality}
We will now present an approach to describe causal models as probabilistic models implemented by Turing machines (TM).
While this idea was first proposed by \citet{janzing2010causal} for deciding the causal graphs among strings, our approach differs in that (1) our complexity score is fundamentally different from their Postulate 7; (2) we focus on the complexity of probabilistic models implemented by a specific class of Turing machines, see~\cref{sec:related_work} and \cref{sec:remark_AMC} for details.

To understand our intuition, suppose we have datasets from different countries on physical health $X$ and sleep quality $Y$. To represent the distribution of these datasets, one can (a) use probability tables for categorical values or binned continuous values; (b) train linear models or nonparametric models such as neural networks; (c) use one of the above methods, but separately over $\PP^i(X)$ and $\PP^i(Y|X)$ or (d) separately over $\PP^i(Y)$ and $\PP^i(X|Y)$. (a-d) might express the same distribution, but we consider them different algorithms or Turing machines. The coding lengths for these models are different. Therefore, we can distinguish (a-d) when we compress the data and take the codebook length into account.

Another example that reflects our motivation is the following: LLMs show impressive performance on the questions that require the use of composable mechanisms such as physical laws, reasoning, and mathematics, but their objective in pre-training is only minimizing cross-entropy of next word given the context \citep{gupta2023context}. By minimizing cross-entropy, can a learned model store and call some reusable mechanisms (conditional probability to certain precisions) just like copying the laws in different environments? Does compression mean that LLMs compress many similar scenarios into sparse mechanisms? To illustrate this point, consider two models $A,B$ which are indistinguishable in inputs and outputs, both compute a function $F:(i,x)\mapsto f(x)$ for all $i$. On input $(1, x)$, $A$ uses a neural network for all $x$ and outputs $f(x)$, and on input $(2, x)$, $A$ copies the neural network $f$ as $\text{COPY}(f)(x)$, where $\text{COPY}(f)$ can be stored as an address in memory pointing to the neural network $f$. Here, 1 or 2 denote the context or environment. B utilizes a sophisticated neural network to compute $F$. Since A and B are undistinguishable as functions, their cross-entropies are also undistinguishable (recall that entropy is the minimal length of encoding iid data \textit{given} a model). Consequently, minimizing cross-entropy does not ensure the presence of any sparse structure within the model itself. We have to go beyond entropy and identifiability, by taking model length into account.

The idea of “copying mechanisms” is the motivation behind \cref{sec:algo_causality} (step 2 and 3 in \cref{def:cond_feat_mechanism_program}) and \cref{prop:SMS}. Instead of treating all the parameters equally, we can decompose a model into some computation steps. First, the model creates some “raw mechanisms”(\cref{def:prob_mech_map}) without assigning them to a variable in $\mathcal{X}^d$, for example ``something falls at the acceleration $9.8m/s^2$''; second, the model apply the raw mechanisms to some concrete objects or values: “An apple in Europe falls at the acceleration $9.8m/s^2$”, “An iron ball in Asia falls…”. The raw mechanism ``something falls at the acceleration $9.8m/s^2$'' is copied multiple times and “something” is replaced by objects in different places or contexts by a feature mechanism (\cref{def:feature_mechanism}), such as ``apple/iron balls in Europe/Asia $\mapsto$ something'' (\cref{eg:CFMP} Step 2). The gravitational acceleration in Europe and Asia are slightly different, but compression will decide whether saving one or two different mechanisms in the model is better, by balancing the model coding length and data-to-model coding length. The model that compresses the Internet optimally would save the variables in the following way: ``\textbf{if} something is on earth, \textbf{then} it falls at the acceleration of $9.8m/s^2$''. Algorithmic causal statements are such reusable ``if...then...'' judgments that emerge from compression.

It is important to notice that there is no interventional data or identifiability here: compression prefers to put "on earth" as a cause of ``acceleration value'', not because we did single-node interventions on those two variables respectively.
Available training data for LLMs does not include intervention targets or environment variables. The data is multi-env, in the sense that contexts are always different in each data sentence (country, temperature, etc.), but there is only statistical dependence (correlation) in the data. Since there are no intervention targets or hard constraints over model class, no identifiability in causal literature is applicable. However, human and compression algorithms can still believe that it is ``on earth'' that ``causes'' the "acceleration value", not the other way around. The rest of the paper explains how compression can prefer one causal statement over others.

\begin{table}[ht]
\centering
\begin{tabular}{|c|c|c|}
\hline
      & \thead{Identifiable, Pearl's causality}  & \thead{Algorithmic causality} \\ \hline
Model class & \makecell{A class of CBNs/SCMs\\ which computes not all possible \\(multi-env) joint distributions}  & \makecell{A class of CFMPs (\cref{def:cond_feat_mechanism_program})/TMs\\ which computes all possible\\ (multi-env) joint distributions}  \\ \hline
\makecell{Goal of model\\selection} & \makecell{Recover the ground truth graph \\ up to certain symmetries by \\ minimum cross-entropy (\cref{lem:IdentminCE})} & \makecell{Ground truth is ill-defined. Even\\bivariate identifiability is impossible.\\Just select the model that minimizes\\ Kolmogorov/ FC complexity\\(\cref{def:FC_complexity}) among CFMPs}\\ \hline
\makecell{Is there model\\pre-selection\\before training}  & \makecell{Yes. Hard priors (including\\ intervention targets) are needed\\ on probabilistic model classes;\\ otherwise no identifiability beyond\\ Markov equivalence class,\\ even if we have multi-env datasets.}  & \makecell{No. Constraining the class\\of CFMPs does not necessarily\\ constrain the class of probabilistic\\ models that those CFMPs can\\ compute.}   \\ \hline
\makecell{Subjectivity of\\ model selection} & \makecell{The probabilistic model class where\\ the groud truth is believed to live}& \makecell{Choice of UFCC (\cref{def:UFCC})\\ on which FC complexity is based}\\ \hline  
\end{tabular}
\caption{Comparison between the learning of identifiable causal models and our algorithmic causal models.}
\end{table}

\begin{definition}[informal, algorithmic causality]\label{def:algo_causality_informal}
    Consider a class of Turing machines where each $T$ halts on any input in $\Xcal^d:=(\Bcal^m)^d$ (see \cref{def:discrete_proba_space}) and outputs a codeword for $x$ or outputs $\PP(x)$ for a certain distribution $\PP$. Suppose all $T$ can simulate certain subprograms in the form of ``If $X_i=\dots$ then $X_j=\dots$''. We say that according to a model selection method (e.g. compression), \emdef{$X_i$ algorithmically causes $X_j$ locally} at $x\in \Xcal^d$ if the method selects a Turing machine $T$ such that, given the input $x$, among all the subprograms that $T$ simulates there exists one in the form of ``If $X_i=\dots$ then $X_j=\dots$'' and there does not exist the opposite. If such a statement holds for all $x\in \Xcal^d$, we say $X_i$ \emdef{algorithmically causes} $X_j$.
\end{definition}

We emphasize that algorithmic causality is a property of the selected Turing machine w.r.t.\ a model selection method, not a property of the data distribution computed by a Turing machine. Notice that in identifiability research, the decision on causal graphs also depends on the model selection method (e.g. choice of the model class). In the following, we are interested in those model selection methods that do not constrain the joint distribution class.

A formal definition for a special case is provided below (\cref{def:algorithmic_causality_CFMP}). We will see in~\cref{sec:learn_compression} that compression prefers selecting a Turing machine that says ``$X_i$ causes $X_j$'' if the estimated conditional probability $\PP(X_j|X_i)$ is approximately invariant across environments. Our model selection strategy is to constrain the class of Turing machines but not the distribution class that our Turing machine class can compute.
We consider discrete random variables, but it is a standard fact that they can approximate all absolutely continuous distributions with compact support to arbitrary precision:
\begin{definition}\label{def:disc_distr_cylinder}
    A \emdef{discrete distribution} $\PP$ supported in the rational number space $(\QQ\cap [0,1))^D$ is a function $(\Bcal^*)^D \to \Bcal^*$, where $\Bcal=\{0,1\}$, and $\Bcal^*$ is the set of binary sequences of arbitrary finite length, which is isomorphic to $\QQ\cap [0,1)$ by the canonic dyadic expansion: $b_1 b_2\dots \in\Bcal^*$ is equivalent to $\sum_{i=1}^\infty \frac{b_i}{2^i}$.\footnote{For simplicity of the reading, in this paper, we do not allow deterministic distributions, i.e., there exists $x\in (\Bcal^*)^D$ such that $\PP(x)=1$. In fact, to represent all values in $\QQ\cap [0,1]$ with precision $n$ with the value $1$ included, $n+1$ bits are necessary instead of $n$ bits.} The \emdef{ $(m,n)$-projection} $\PP^{(m,n)}: (\Bcal^m)^D \to \Bcal^n$ is defined by 
      \begin{equation}
        \PP^{(m,n)}(x)=
        \PP(x|_m)|_n
    \end{equation}
    where $x|_m$ denotes the $m$-length prefix of $x$. We call $m$ the \emdef{precision} of the variables in $\PP$, and $n$ the precision of the probability values in $\PP$.
\end{definition}
We leave the more formal definition of above, using the notion of the \textbf{cylinder}, to~\cref{sec:Discrete_proba}.

\begin{definition}\label{def:discrete_proba_space}
    Given the precision $m$, we denote $\Xcal:=\Bcal^m$. In this paper, we focus on probability distributions on the  space $\Xcal^d=(\Bcal^m)^d$, where any $(x_1,\dots x_d)\in \Xcal^d$ is a concatenation of $d$-many binary sequences of length $m$.
    We call a \emdef{multi-env system} a list of discrete distributions $(\PP^i)_{i \in [I]}$ on $\Xcal^d$. $I$ denotes the number of environments.\footnote{
Below, we use the terms ``distribution'' and ``multi-env system'' 
interchangeably, since the latter can be written as a distribution over $\Xcal \times [I]$. We are not given intervention targets, hence any variable can be the environment label.} Equivalently, we can also write a multi-env system on $\Xcal^{d-1}$ with less than $2^m$ environments as a discrete distribution on $\Xcal^d$.
\end{definition}
\begin{definition}\label{def:compute}
    A Turing machine $T$ is said to \emdef{compute} a function $f: A\subset \Xcal \to \Bcal^*$, if $T$ rejects any input outside $A$ 
    and for all $x\in A$, $T(x)=f(x)$. We define the \emdef{equivalence relation} $\sim$ between two Turing machines $S$ and $T$ if they compute the same function.
\end{definition}
We provide a more detailed introduction on computation theory in~\cref{sec:Computation}, and a more formal version of \cref{def:compute} in \cref{def:compute_formal}.
We now introduce algorithmic models that can compute certain probability distributions. For a review of related approaches we refer to Section~\ref{sec:related_work}.
\begin{definition}\label{def:prob_mech_map}
    A \emdef{probabilistic mechanism} is a Turing machine that computes
    the following discrete functions, which we call \emdef{probabilistic maps}: $f: \Bcal^{v_f} \times \Bcal^{c_f}\to \Bcal^{n}$. We call $v_f$ the bit length for the \emdef{value variable} of $f$, $c_f$ the bit length for the \emdef{conditional variable} of $f$.
    Value and conditional variables are just the placeholders in the input of the probabilistic mechanism. The output $b_1\dots b_n \in\Bcal^n$ is equivalent to a rational number in $[0,1)$, i.e. $\sum_{i=1}^n \frac{b_i}{2^i}$.

    For a set of probabilistic mechanisms $\Pcal$, we define its corresponding \emdef{probabilistic map set} $\overline{\Pcal}:= \Pcal / \mathord{\sim}$. 
\end{definition}
The key idea is that we do not sample from a distribution stochastically; we consider a probabilistic map as a deterministic function that maps points to values in $[0,1)$. A probabilistic mechanism is a Turing machine that computes such a function. In this paper, a \emdef{mechanism} always denotes a Turing machine instead of a function or distribution.

\begin{definition}\label{def:feature_mechanism}
    Given the dimension $d$ and precision $m$, we define the \emdef{feature mechanism set} $\Phi$ as a set of Turing machines that compute functions $\Xcal^d\to \Bcal^*$, which we call \emdef{feature maps}. We define its corresponding \emdef{feature map set} $\overline{\Phi} := \Phi / \mathord{\sim}$.
\end{definition}
\begin{figure}[t]
    \centering
    \includegraphics[width=\textwidth]{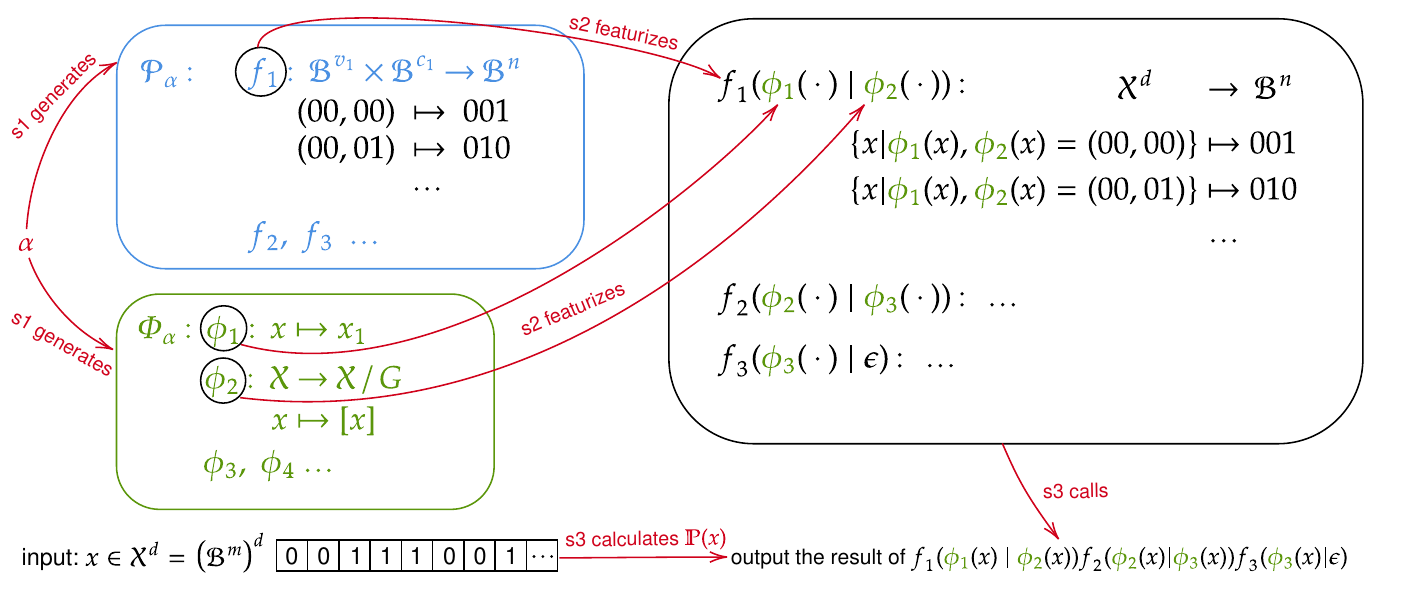}
    \caption{Illustration of a CFMP (\cref{def:cond_feat_mechanism_program}). A CFMP $\alpha$ is a Turing machine that sequentially proceeds in three steps in red given any input in $\Xcal^d$. Probabilistic mechanisms are blue and feature mechanisms are green. $\epsilon$ denotes the empty string. Before reading the input tape, $\alpha$ proceeds in two steps: generates $\Pcal_\alpha, \Phi_\alpha$, and featurizes the probabilistic mechanisms. In the third step, $\alpha$ multiplies the conditional probabilities it needs for calculating $\PP(x)$ and marginalizes over latent variables if there are hidden-variable mechanisms. We emphasize that this figure, or CFMP, is not a process of \textit{learning}, but just a \textit{model} in the model class where we proceed model selection.}

    \label{fig:CFMP}
\end{figure}
In this paper, we focus on one class of Turing machines that we term CFMPs, and we show in~\cref{sec:case_studies} that it aligns compression with causality and symmetry.

\begin{definition}\label{def:cond_feat_mechanism_program}
    For any discrete distribution $\PP$, a \emdef{conditional feature-mechanism program (CFMP)} $\alpha$ is a Turing machine that computes $\PP$ by doing the following steps (\cref{fig:CFMP}):

        \begin{enumerate}
            \item Without reading the input tape, $\alpha$ generates $\Pcal_\alpha$ and $\Phi_\alpha$ in the memory, where $\Phi_\alpha$ is a feature mechanism set, and $\Pcal_\alpha$ is a probabilistic mechanism set.
            \item Without reading the input tape, $\alpha$ featurizes the probabilistic mechanisms: for certain $f\in \Pcal_\alpha$, $\alpha$ selects\footnote{This selection procedure determines a function: $\{\text{featurized mechanisms generated in Step 2\}} \to \Pcal_\alpha$. Note that the bit length is not intrinsic to the function, but depending on the CFMP that computes this function.} certain $\phi:\Xcal^d\to \Bcal^{v_f}$ and $\psi:\Xcal^d\to \Bcal^{c_f}$ in $\Phi_\alpha$, and stores either the Turing machine which computes $x \mapsto f(\phi(x)|\psi(x))$ (termed ``no-hidden-variable mechanism'') or $(x,x') \mapsto f(\phi(x,x')|\psi(x,x'))$ (termed ``hidden-variable mechanism'').
            Namely, while $f$ maps  $\Bcal^{v_f} \times \Bcal^{c_f} \to \Bcal^n$, 
             the featurized mechanism $f_i$ maps $\Xcal^d \to \Bcal^n$ or $\Xcal^d \times \Xcal^d \to \Bcal^n$. 
            \item The set of featurized mechanisms is written in the memory. 
            For any input $x\in \Xcal^d$, $\alpha$ selects\footnote{In principle the selection of featurized mechanisms in Step 3 depends on $x\in\Xcal^d$. Namely, this selection procedure determines a function $\Xcal^d\to \{\text{featurized mechanisms generated in Step 2\}}$.\label{ftn:CFMP_select}} and computes the multiplication of a sequence of featurized mechanisms evaluated on $x$: 
            \begin{enumerate}
                \item if only no-hidden-variable mechanisms are selected for $x$:
                $$f_1(\phi_1(x)|\psi_1(x))\dots f_k(\phi_k(x)|\psi_k(x))$$
                for certain $k$, then compute and output the result.
                \item if there exist hidden-variable mechanisms in the formula
                $$f_1(\phi_1(x,x')| \psi_1(x,x'))\dots f_k(\phi_k(x,x')| \psi_k(x,x'))$$
                (where $x'$ is a formal placeholder in the formula without value assignment) for certain $k$, then marginalize over $x'\in \Xcal^d$ and output the result.
            \end{enumerate} 
    \end{enumerate}
\end{definition}

\begin{example}[multi-env CBN]
Consider a multi-env system: $\PP(X_1,X_2,E=1)=\PP(X_2|X_1)\PP(X_1|E=1)\PP(E=1)$, $\PP(X_1,X_2,E=2)=\PP(X_2|X_1)\PP(X_1|E=2)\PP(E=2)$. 

Now we give an example of CFMP $\alpha$ that computes this system. Step 1 in Def 9, $\alpha$ generates the following:
Probabilistic mechanisms (\cref{def:prob_mech_map}): $f_1(\cdot|\cdot)$, $f_2(\cdot|\cdot)$, $f_3(\cdot)$. 

Feature mechanisms (\cref{def:feature_mechanism}): $\phi_1: (x_1, x_2 ,e)\mapsto x_1$, $\phi_2: (x_1, x_2 ,e)\mapsto x_2$, $\phi_3: (x_1, x_2 ,e)\mapsto e$.

Step 2, $\alpha$ featurizes the mechanisms: $f_1 \mapsto f_1(\phi_2(\cdot)|\phi_1(\cdot))$, $f_2 \mapsto f_2(\phi_1(\cdot))$, $f_3 \mapsto f_3(\phi_3(\cdot))$.

Step 3, given any $z=(x_1,x_2,e)$, compute the probability value $$\PP(z)= f_1(\phi_2(z)|\phi_1(z))f_2(\phi_1(z))f_3(\phi_3(z)).$$

\end{example}

\begin{example}\label{eg:CFMP}
    The following examples show that the class of CFMPs can model a wide range of probabilistic models, which turn into classes of Turing machines instead of distributions. All of the examples except causal representation learning models use only no-hidden-variable feature mechanisms. In the following examples, $\PP$ denotes respectively a discrete (multi-env) system that the CFMP computes:
    \begin{enumerate}[itemsep=0pt]
        \item A \emdef{causal Bayesian network (CBN)} model computing $\PP$ on $\Xcal^d$ is a CFMP $\alpha$ such that
        \begin{enumerate}
            \item the feature map set is $\overline{\Phi_\alpha} \subset \{(\cdot)_A | A\subset 2^{[d]}\}$, projections over coordinates.
            \item In step 3, for all $x\in\Xcal^d$, $\alpha$ computes $\PP(x)$ using the same set of featurized mechanisms: $f_1(x_{A_1}|x_{B_1})\dots f_k(x_{A_k}|x_{B_k})$, such that $\bigsqcup_{i=1}^k A_i=[d]$ and $A_i \cap B_i = \emptyset \quad \forall i$ and $B_i \subset \cup_{j=1}^{i-1} A_j \quad \forall i$.
        \end{enumerate}
        \item A \emdef{context-specific Bayesian network} model~\citep{boutilier1996context} computing $\PP$ is a CFMP $\alpha$ that has the same definition as CBN except (b) in step 3 of CBN: for different points $x,x'\in\Xcal^d$, $\alpha$ is allowed to use different sets of featurized mechanisms depending on $x,x'$, see footnote \ref{ftn:CFMP_select}.
        \item A \emdef{causal representation learning} model computing $\PP$ on $\Xcal^d$ is a CFMP $\alpha$ such that
        \begin{enumerate}
            \item  in $\Pcal_\alpha$ there exists a probabilistic mechanism $f: \Xcal^d \times \Zcal^d \to \Bcal^n$, where $\Zcal^k:=(\Bcal^{m'})^k$ denotes a space for hidden variables, of arbitrary precision $m'$ and dimension $k$.
            \item There exists a feature mechanism $\psi: \Xcal^d \to \Zcal^k$ that featurizes $f$ into a hidden-variable mechanism $(x,x')\mapsto f(x | \psi(x')): \Xcal^d \times \Xcal^d\to \Bcal^n$.
            \item In step 3, for all input $x\in\Xcal^d$, $\alpha$ outputs $\sum_{x'\in\Xcal^d} f(x | \psi(x')) g(\psi(x'))$, where $g$ is any featurized mechanism, or a product of different mechanisms.
        \end{enumerate}
        \item A \emdef{$G$-invariant learning} model computing $\PP$ on $\Xcal^2$ is a CFMP $\alpha$ where $G$ is a group that acts on $\Xcal$ and\footnote{For simplicity of presentation, we give examples that compute one-dimensional invariant and equivariant models. One can easily generalize examples 4, 5 and 6 to higher dimensions.}
        \begin{enumerate}
            \item there exists $\phi: \Xcal \to \Xcal / G$ in $\Phi_\alpha$ partitioning the orbits of $\Xcal$ under the action of $G$. In addition, $\Phi_\alpha$ contains projections on two coordinates $\pi_1$ and $\pi_2$. $\Phi_\alpha$ also contains $\phi \circ \pi_1$.\\
            There exists $f_1,f_2\in\Pcal_\alpha$ such that $f_1(\pi_2(x)|\phi\circ \pi_1(x))=\PP(x_2|\phi(x_1))$, and $f_2(\pi_1(x))=\PP(x_1)$ for all $x\in\Xcal^2$.
            \item In step 2, $\alpha$ featurizes the probabilistic mechanisms: $f_1\mapsto f_1(\pi_2(\cdot)|\phi\circ \pi_1(\cdot))$, $f_2\mapsto f_2(\pi_1(\cdot))$, and $f_3\mapsto f_3(\pi_1(\cdot))$.
            \item In step 3, for all $x\in \Xcal^2$, $\alpha$ computes $\PP(x)=\PP(x_2|\phi(x_1))\PP(x_1)=f_1(\pi_2(x)|\phi\circ \pi_1(x)) f_2(\pi_1(x))$.
        \end{enumerate}
        \item A \emdef{transitive $G$-equivariant learning} model computing $\PP$ on $\Xcal^2$ is a CFMP $\alpha$ such that
        \begin{enumerate}
            \item in $\Pcal_\alpha$ there exists a probabilistic mechanism $f: \Xcal \to \Bcal^n$ modeling a function $x \mapsto\PP(\Tilde{x}|g.x)$ where $\Tilde{x}$ is an arbitrary fixed point in $\Xcal^d$. For each $(x,x')\in \Xcal \times \Xcal$, there exists $g\in G$ such that $x=g.\Tilde{x}$ and $\PP(x|x')=f(x|x')=f(g^{-1}.x|g.x')= f(\Tilde{x}|g.x')=\PP(\Tilde{x}|g.x')$.
            \\There exists $|G|$-many feature mechanisms in $\Phi_\alpha$ computing the group actions on $\Xcal$: for all $g\in G$, $\phi_g: x\mapsto g.x$.
            \item In step 2, $\alpha$ generates $|G|$-many featurized mechanisms using $f$ and $(\phi_{g^{-1}}, \phi_g)$ for all $g\in G$: $f \mapsto f(\phi_{g^{-1}}(\cdot) |\phi_g (\cdot))$. $\alpha$ also generates a featurized mechanism $h(\phi_g \circ \pi_1(\cdot))$ such that $\PP_1(\phi_g(x_1))=h(\phi_g \circ \pi_1(x))$
            \item In step 3, $\alpha$ only uses those featurized $f(\phi_{g^{-1}}(\cdot) |\phi_g (\cdot))$ for computing the value variable $y\in\pi_2(\Xcal^2)$: for all $x\in\Xcal^2$, there exists a featurized mechanism $h$ and $g\in G$ such that \\
            $\PP(x)= \PP_2(\phi_{g^{-1}}(x_2)|\phi_g(x_1))\PP_1(\phi_g(x_1))
            =f(\phi_{g^{-1}}\circ \pi_2(x)|\phi_g \circ \pi_1(x)) h(\phi_g \circ \pi_1(x))$.
        \end{enumerate}
        \item A \emdef{$G$-equivariant learning} model computing $\PP$ on $\Xcal^2$ is a CFMP $\alpha$ that has the same definition as the transitive $G$-equivariant learning model except that in (a):
        \begin{enumerate}
            \item in $\Pcal_\alpha$ there exists a probabilistic mechanism $f: \Xcal / G \times \Xcal \to \Bcal^n$ where each element in $\Xcal / G$ is a fixed representative element of an equivalence class.\footnote{Different from the common usage of the quotient in algebra, here each representative in $\Xcal / G$ is a given point in $\Xcal$.} For each $(x,x')\in \Xcal \times \Xcal$, there exists $g\in G$ and $\Tilde{x}\in\Xcal / G$ such that $x=g.\Tilde{x}$ and $\PP(x|x')=f(x|x')=f(g^{-1}.x|g.x')= f(\Tilde{x}|g.x')=\PP(\Tilde{x}|g.x')$.
        \end{enumerate}
        \item A \emdef{statistical density estimator} computing $\PP$ on $\Xcal^d$ is a CFMP $\alpha$ that
        \begin{enumerate}
            \item in $\Pcal_\alpha$ there is only one Turing machine, which computes the distribution $\PP: \Xcal^d\to \Bcal^n$.
            \item In $\Phi_\alpha$ there is only one Turing machine, which computes identity $\textrm{Id}: \Xcal^d \to \Xcal^d$.
            \item In step 2, $P$ is featurized trivially by identity feature map, so $\PP$ is not changed.
            \item In step 3, for all $x\in \Xcal^d$, $\alpha$ outputs $\PP(x)$.
        \end{enumerate}
    \end{enumerate}
\end{example}
Now we concretize the general definition of algorithmic causality (\cref{def:algo_causality_informal}) for a class of Turing machines, i.e., CFMPs. First, let us focus on the class of CFMPs that compute only CBNs. This is the most familiar case to the readers in causality.
\begin{definition}[Algorithmic causality in CBN]\label{def:algorithmic_causality_CBN}
    Given a class of CFMPs that compute the causal Bayesian network models, for $A, B \subseteq [d]$ with $A\cap B=\emptyset$, we say that according to a model selection method\footnote{Such as minimizing the finite codebook complexity, which we introduce in~\cref{sec:learn_compression} and~\cref{def:FC_complexity}.}, \emdef{$X_A$ algorithmically causes $X_B$ locally at $x\in \Xcal^d$} if the method selects a CFMP $\alpha$ such that, in the third step in $\alpha$ the featurized mechanisms for $x$ include one mechanism $f(X_B|X_A)$ in which $X_B$ is a set of value variables and $X_A$ is a set of conditional variables (\cref{def:prob_mech_map}), and do not include mechanisms with the opposite direction $f(X_A|X_B)$.
    Similarly, we say that \emdef{$X_A$ algorithmically causes $X_B$} if the above assumption holds for all $x\in \Xcal^d$.
    \end{definition}

We observe that the causal variables $X_i$ are in fact $\pi_i(X)$, the $i$-th projection over $X$. We generalize the definition above to the case of arbitrary feature map:
\begin{definition}[Algorithmic causality]\label{def:algorithmic_causality_CFMP}
    Given a class of CFMPs, we say that according to a model selection method, \emdef{$\phi(X)$ algorithmically causes $\psi(X)$ locally at $x\in \Xcal^d$} if the method selects a CFMP $\alpha$ such that, in the third step in $\alpha$ the featurized mechanisms for $x$ include one mechanism $f(\psi(X)|\phi(X))$ in which $\psi(X)$ is a value variable and $\phi(X)$ is a conditional variable (\cref{def:prob_mech_map}), and do not include mechanism with the opposite direction.
    Similarly, we say that \emdef{$\phi(X)$ algorithmically causes $\psi(X)$} if the above assumption holds for all $x\in \Xcal^d$.
    
    Moreover, if the class of CFMP allows hidden-variable mechanisms (\cref{def:cond_feat_mechanism_program}, 2.) in certain CFMPs and if the method selects a CFMP $\alpha$ such that in the third step in $\alpha$ the featurized mechanisms for $x\in\Xcal^d$ include one hidden-variable mechanism $(x,x')\mapsto f(\phi(x)|\psi(x'))$ and do not include the mechanisms with the opposite direciton, then we say that the hidden variable $\psi(X')$ algorithmically causes $\phi(X)$ locally at $x\in\Xcal^d$, and $\psi(X')$ algorithmically causes  $\phi(X)$ if it holds for all $x\in\Xcal^d$.
\end{definition}
For example, in causal representation learning we are typically interested in the latent causal graph, which is in our language the causal relationships between $(\pi_i \circ \psi(X'))_{i\in [k]}$, where $k$ denotes the dimension of the latent space. Our definition of causality is flexible enough to model the causal relationships between \textit{imagined} variables, i.e. between $\phi(X)$ and $\psi(X)$ with $\phi,\psi$ being any feature mechanisms. Different from the measure-theoretic foundation of interventional causality \citep{park2023measure}, we do not need to predetermine the causal variables and their spaces. Our choice of defining causal variables can emerge from the model selection through the choice of feature mechanisms. In fact, the dimension $d$ and precision $m$ in $\Xcal$ are only used to illustrate the abstract concepts of \cref{def:cond_feat_mechanism_program}; equivalently, one can replace $\Xcal^d$ by a binary input tape of arbitrary fixed length, since in a computer or UFCC there are only processes of binary variables at the basic level.
\section{Learning algorithmic causality by compression}\label{sec:learn_compression}

In~\cref{sec:algo_causality}, we introduced a computational model that can compute a wide range of probabilistic models. Recall that in our problem setting, no identifiability beyond the Markov equivalence class is possible for probabilistic models. However, model selection beyond the Markov equivalence class is possible for our computational model, e.g., based on the coding length of the model. This model selection is crucial because it determines algorithmic causations (Definition~\ref{def:algorithmic_causality_CFMP}).
Since the goal of identifiability is to justify that the model that minimizes the multi-env cross-entropy should be the preferred model (in that case, the ground truth), we generalize that principle:
\begin{principle}\label{priniple:compression}
    Given (multi-env) datasets, the preferred (causal) models to select (no matter algorithmic or probabilistic) are the ones that minimize the bit length of a sender's message enabling the receiver to reconstruct the multi-env datasets.
\end{principle}
Notice that this principle holds in classical causal discovery whether the model class is identifiable or not because the model that minimizes cross-entropy also maximizes likelihood, which is preferred over those that are less likely.
As discussed in~\cref{sec:comment_entropy}, entropy is the minimal average data-to-model coding length \textit{given a codebook}, therefore it is not the \textit{overall} bit length needed by a sender. Algorithmic causality makes the codebook length unignorable because the length of CFMP is upper boundable (sometimes computable) and unignorable when the data is finite. 
The rigorous \textit{overall} bit length that a sender needs for lossless encoding of a data sequence $x$ is called \textit{Kolmogorov complexity} $C(x)$. 
\subsection{Kolmogorov complexity}\label{sec:KC}
We review Kolmogorov complexity, which led to the idea of two-part code~\citep{li2019introduction}. This inspired our idea of finite codebook complexity. We leave some definitions in computation theory in \cref{sec:Computation}.
\begin{definition}\label{def:KC}\citep{kolmogorov1968three, li2019introduction}
    (First version) For any $x\in \NN$, the \emdef{Kolmogorov complexity} of $x$ w.r.t. the universal Turing machine $U$ is defined as
    \begin{equation}
        C_U(x)=\min_{T\in \text{\{Turing machines\}}} \{l_U(T)| U(T)=x\}
    \end{equation}
    where $l_U$ is a mapping from the class of all Turing machines to $\NN$ such that for each $n\in \NN$ there are less than $2^n$ Turing machines $T$ such that $l_U(T)\leq n$.

    (Second version) equivalently, the \emdef{Kolmogorov complexity} of $x$ w.r.t. the universal Turing machine $U$ can also be defined as
    \begin{equation}
        C_U(x)=\min_{n\in\NN} \{l_U(n)| U(n)=x\}
    \end{equation}
    where $l_U$ is a monotonically increasing map from $\NN \to \NN$ such that for any $n\in \NN$, $l_U(2^n) \leq n$. 
    \end{definition}
In the second version of the definition, the input of $l_U$ is not a Turing machine, but an index of a Turing machine. It is important to note that each universal Turing machine (UTM) defines a computable bijective mapping $\NN \to \{\text{Turing machines}\}$: $U(0)=T_0, U(1)=T_1, \dots$,\footnote{By convention, $U(0)$ is often defined as $U$ itself \citep{li2019introduction}.} which is called an \emdef{effective enumeration of TMs}. A UTM does not have to take the literal description of a Turing machine as input. There exists a UTM $U$ for which a TM $T$ with 5 states has the length $l_U(T)=1$. In \cite{li2019introduction} and much of the literature, people use $l$ instead of $l_U$, which can be somewhat confusing because they implicitly assume that readers are aware that each UTM defines a different effective enumeration over all TMs. In this paper, we use $l$ to denote the \emdef{literal length function} that maps a natural number $n$ (or its equivalent binary string $B(n)$, see \cref{def:binary_lexico_code}) to $\lfloor \log_2 (n + 1) \rfloor$, and we use $l_U$ in both versions, which can be distinguished automatically by their input.
\begin{lemma}\citep{kolmogorov1968three, li2019introduction}\label{lem:UTM_optimal}
    There is an additively optimal universal Turing machine $U$, i.e. for all UTM $V$ and all x, $C_U(x)\leq C_V(x) + O(1)$.
\end{lemma}

\begin{proof}
    We will construct a UTM $U$. It needs inputs in the form
    $$\langle n,p\rangle=\underbrace{11\dots1}_{\begin{array}{c}l(n)\text{ times}\\\end{array}}\quad0~B(n) B(p)$$
    where \emdef{$\langle \cdot, \cdot \rangle: \NN\times \NN \to \Bcal^*$ is the self-delimiting concatenation (\cref{def:self_delim_code})}\footnote{Its use is simply to ensure that the concatenated numbers $(np)$ are uniquely identified.}, and $B(\cdot)$ is the \emdef{binary lexicographic code} defined in \cref{def:binary_lexico_code}.
    $U$ lists all TMs $T_0, T_1, T_2, \dots$. Define $U(\langle n,p\rangle):= T_n(p)$. Namely, the UTM $U$ first receives the self-delimiting code for $n$, then it simulates the Turing machine $T_n$, on the input number $p$.

    For any UTM $V$, there exists $n$ such that $V=T_n$. Suppose $p^*\in \argmin_{p\in \NN}\{l(p): T_n(p)=x\}$, then $C_U(x)\leq l(\langle n,p^*\rangle) = 2 l(n)+1 + l(p^*) = 2 l(n)+1 + C_{T_n}(x) = c_{U,V} + C_V(x)$ where $c_{U,V}:= 2 l(n)+1$ only depends on $U$ and $V$.

\end{proof}
This Lemma implies that for all such additively optimal TM, we can choose any of them with a coding length different by $O(1)$.

\begin{definition}\citep[Def. 2.1.1]{li2019introduction}
    Given an additively optimal UTM $U$, given $x,y\in \NN$, the conditional Kolmogorov complexity $C_U(x|y)$ is defined as 
\begin{equation}
    C_U(x|y)=\min_{p\in \NN}\{l(p): U(\langle y,p \rangle)=x\}
\end{equation}
\end{definition}

\begin{lemma}[Kolmogorov complexity can be written as two-part code~\citep{li2019introduction}]\label{lem:two-part_code}

Suppose $U$ is an additively optimal UTM taking inputs $\langle n,p\rangle$ as in the proof of~\cref{lem:UTM_optimal}, then
\begin{align}
C_U(x)&=\min\{2l_U(T) +l_U(p):T(p)=x\}\\
& =\min\{2l_U(T) +C_U(x|T):T\in\{T_0,T_1,\ldots\}\}
\end{align}
where $2l_U(T) (+1)$ is the binary bit length of the \emdef{self-delimiting code} (\cref{def:self_delim_code}) of $T$ in the effective enumeration of $U$.
\end{lemma}
\begin{proof}(Adapted from \cite[2.1.1]{li2019introduction})
     By definition, $C_U (x)=\min\{l(q): U(q)=x\}$, with $q=\langle n,p\rangle$ as in \cref{lem:UTM_optimal}. If $q^*$ is a number such that $l(q^*) = C_U(x)$, then there exists $n^*, p^*$ s.t. $C_U (x) = l(q^*) = 2l(n^*) +1 + l(p^*) = 2l_U(T_{n^*})+1+ l(p^*)$ and
     $T_{n^*}(p^*)=x$, where $l_U(T_{n^*})=l(n^*)$ is the length of the index of the Turing machine $T_{n^*}$ in $U$ and its self-delimiting coding length is $2l(n^*)+1$ according to \cref{def:self_delim_code}.
\end{proof}

Intuitively, the $l_U(T)$ part of the code squeezes out the regularities in x. $C_U(x|T)$ is irregularities, or random aspects, of x relative to that Turing machine. Since most strings are algorithmically random or incompressible \citep[Thm. 2.2.2]{li2019introduction}, minimum cross-entropy is often the shortest coding length for a sequence $x$ \textit{given} the distribution computed by $T$ that achieves the minimum cross-entropy \citep[Thm. 8.1.2]{li2019introduction}. This idea has led to the two-part code objective in statistical inference \citep{grunwald2007minimum}, where the irrgularity part is upper bounded by the Shannon code of an iid sequence $c(x_1\dots x_n):=\sum_{i=1}^n -\log \PP(x_i)$ \textit{given a codebook}, but how to bound the regularity part, $l_U(T)$, to our knowledge, has not been discussed in previous literature. 
In the following, we give an upper bound of Kolmogorov complexity and the regularity part $l_U(T)$ in the discrete finite sample space.

\subsection{Finite codebook complexity}\label{sec:FC_complexity}

After decomposing the Kolmogorov complexity of a multi-env dataset $C_U(x_1, \dots x_n)$ into a two-part code $l_U(T)+C_U(x_1, \dots x_n|T)$, there is still an uncomputable part $l_U(T)$.  We can construct an upper bound of $l_U(T)$ by constraining the Turing machine class, without constraining the distribution class or codebook class that our Turing machines can compute.
\begin{definition}\label{def:FC_FCM}
    We say that an injective function $g: \Acal\subset \Xcal^d \to \Bcal^*$ is a \emdef{finite codebook} if the set $g (\Acal)$ is prefix-free. We say that a Turing machine $T$ is a \emdef{finite coding mechanism (FCM)} if it computes a finite codebook.
\end{definition}

\begin{definition}\label{def:UFCC}
    Given the dimension $d$ and precision $m$ of $\Xcal^d$, we say that a Turing machine $V$ is a \emdef{universal finite codebook computer (UFCC)}\footnote{For simplicity of presentation we do not set $(m,d)$ as part of the input in a UFCC, while in the general sense, a UFCC should take $m,d,k,p$ as input. Suppose we are given a recursively enumerable set of FCMs for each $(m,d)$, then we can construct a UFCC inputting the self-delimited $(m,d,k,p)$, i.e. $\langle m, \langle d, \langle k,p \rangle\rangle\rangle\rangle$.} if
    \begin{enumerate}
        \item $V$ takes input $\langle k,p \rangle$, where $k$ is the index of an FCM in a decidable set\footnote{Namely, the set is recursively enumerable (r.e.) and its complement set in the set of all Turing machines is also r.e.} of FCMs; $p$ is a natural number, which is equivalent to a binary string $B(p)$ (see \cref{def:binary_lexico_code}); same as \cref{lem:two-part_code}, $\langle \cdot, \cdot \rangle$ is the self-delimiting concatenation.
        \item for any finite codebook $g: \Acal\subset \Xcal^d\to \Bcal^*$, there exists $k$ such that $V(\langle k, \cdot \rangle)$ is an FCM computing $(g^*)^{-1}$, which is a partial (i.e., not everywhere defined) function $\Bcal^* \to \Acal^*$, and $g^*$ is the extension of the codebook $g$ (see \cref{def:extension_codebook}, $(g^*)^{-1}$ is well-defined because $g$ is prefix-free);
        \item $V(\langle k,p \rangle)$ is computed by decoding the binary string $B(p)$ using $V(\langle k, \cdot \rangle)$.
    \end{enumerate}

\end{definition}

\begin{definition}\label{def:FC_complexity}
    Given a UFCC $V$, for any finite codebook $g$, the \emdef{finite codebook (FC) complexity} is defined as
    \begin{equation}
        C^{FC}_V(x_1\dots x_n) := \min_{(k,p)\in\NN^2} \{l(\langle k,p \rangle): V(\langle k,p \rangle)=x_1\dots x_n\}.
    \end{equation}
\end{definition}

We now first apply~\cref{lem:UTM_optimal} to upper bound $C_U(x_1\dots x_n)$ by $C^{FC}_V(x_1\dots x_n)$, and then split $C^{FC}_V(x_1\dots x_n)$ into a two-part code using the idea similar to~\cref{lem:two-part_code}.

For any additively optimal (\cref{lem:UTM_optimal}) UTM $U$ and any UFCC $V$, using the same argument as in the proof of \cref{lem:UTM_optimal}, there exists $n$ such that $V=T_n$, so we can bound the Kolmogorov complexity $C_U$ by FC complexity $C^{FC}_V$:
\begin{equation}
    C_U(x_1\dots x_n) \leq C^{FC}_V(x_1\dots x_n) + O(1).
\end{equation}
Replace $C^{FC}_V$ by \cref{def:FC_complexity},
\begin{align}
    C_U(x_1\dots x_n) &\leq \min_{k,p} \{l(\langle k,p \rangle): V(\langle k,p \rangle)=x_1\dots x_n\} + O(1)\\
    &= \min_{T,p}\{2l_V(T) + l(p): T(p)= x_1\dots x_n\} + O(1)
    \label{eq:two-part-objective}
\end{align}
where $2l_V(T) (+1)$ is the self-delimiting code length (\cref{def:self_delim_code}) of the FCM $T$ in the effective enumeration (\cref{def:KC}) of FCMs according to $V$, and $l(p)$ is the literal length of the binary code that $T$ takes as input and thereby outputs $x_1\dots x_n$.

Given a UFCC $V$, given any finite codebook $g: \Xcal^d\to \Bcal^*$ and a number $p$ such that $(g^*)^{-1}\circ B(p)=x_1\dots x_n$, there exists a FCM $T$ that computes $(g^*)^{-1}$, and thus $C_U(x_1\dots x_n)$ is upper bounded by $2l_V(T) + l(p)$. In particular, if we believe that the data $x_1\dots x_n$ is iid sampled from a (multi-env) system or discrete distribution, then we should use a Shannon codebook (i.e. the codeword length for $x$ is $-\log \PP_{\hat{\theta}}(x)$) and a Shannon codeword $p$ (i.e. $p=g(x_1)\dots g(x_n)$), where $\hat{\theta}$ is the maximum likelihood estimator in the class of all the discrete distributions supported in $\Xcal^d$. This is justified by Shannon's source coding theorem: a shortest $p$ when $n\to\infty$ is $g(x_1)\dots g(x_n)$.

There is a trade-off between $l_V(T)$ and $l(p)$: if $n\to \infty$ then $l(p)\to nH(\PP_{\hat{\theta}})$ dominates. In the finite data case $l_V(T)$ (codebook part) is not negligible.

 \vspace{-.05cm}
\begin{definition}\label{def:huffman_mechanism}
    Given $\Acal\subset \Xcal^d$, We say that a Turing machine $\gamma$ is a \emdef{Huffman coding program} if given  any discrete finite distribution $\PP$ supported on $\Acal$ as input it outputs a Turing machine that computes a Huffman code (\cref{def:Huffman_code}) of $\PP$.
\end{definition}

    \vspace{-.3cm}
\begin{example}\label{eg:UFCC}
    Some examples of UFCC:
    \begin{itemize}[itemsep=1pt,topsep=3pt]
        \item There exists a Turing machine $V_\Ccal$ that takes any element in a finite class $\Ccal$ of CFMP as input and combines it with a Huffman coding program to compute a Huffman code, then decodes an integer $p$ which is equivalent to an input binary sequence $B(p)$.
        \item There exists a Turing machine $V_1$ that takes a (multi-env) discrete distribution table as input and combines it with a Huffman coding program to compute a Huffman code, then decodes $p$.
        \item There exists a Turing machine $V_2$ that inputs any FCM as a table and then decodes $p$ using that FCM.
    \end{itemize}
    Given precision $(m,n)$ for $\Xcal^d$, any UFCC can simulate any finite codebook of this precision. By~\cref{cor:correpondance_proba_code}, each prefix code corresponds to a unique probability semi-measure, therefore any UFCC can express any probability measures in $\Xcal^d$ in precision $(m,n)$.
\end{example}
\begin{example}
Note that the following Turing machine is not a UFCC:

    Given $\Xcal^d$ of precision $m$, $V_{\text{halt}}$ inputs two integers:
    \begin{itemize}
        \item an index $k$ of any Turing machine that halts at any point in $\Xcal^d$ and computes a finite codebook;
        \item an integer $p$ which is equivalent to a binary sequence as codeword.
    \end{itemize}
The reason is a direct reduction from the halting problem (\cref{sec:proof_undecidable}):
\begin{restatable}{lemma}{FCMundecidable}\label{lem:FCMundecidable}
    Given $l\in \NN$, the set of Turing machines that halt at any input of length $l$ and compute a finite codebook is undecidable.
\end{restatable}
Therefore, it is necessary to constrain the class of FCMs when defining UFCC, while not constraining the class of finite codebooks or probability distributions that FCMs can compute.
\end{example}

Given the precision $(m,n)$ and dimension $d$, given a multi-env dataset on $\Xcal^d$, we can minimize the upper bound of FC complexity \cref{eq:two-part-objective} in a class of Turing machines or CFMPs. 
For many UFCCs, computing FC complexity is very hard. 
By \cref{cor:correpondance_proba_code},there is a one-to-one correspondence between code length functions and probability distributions.
For those UFCCs that simulate a CFMP followed by a \textit{fixed} Huffman coding program, it suffices to calculate the bit length needed in the CFMP part in order to perform model selection. In \cref{sec:case_studies} we compare the bits needed by different CFMPs under a given UFCC.

\subsection{Comparisons among UFCCs}\label{sec:compare_UFCC}
Are some UFCCs better than others? 
We are not interested in finding the UFCC that achieves the smallest FC complexity for all data, because it is often not computable. Instead, we are interested in finding some UFCCs that are good at model selections, i.e. such a UFCC should not consider all FCMs to be equally preferable to select. 

For the Kolmogorov complexity, the choice of UTM is not important because all the additively optimal UTMs are equivalent up to $O(1)$ (\cref{lem:UTM_optimal}). For UFCC it is not the case.

Consider the following extreme example: a UFCC $U_{\text{unif}}$ first takes $m,d,n$ as input where $m$ denotes the precision of $\Xcal$, i.e. $|\Xcal|=2^m$, and $d$ denotes the number of variables, and $n$ denotes the precision of discrete distribution values. Let $U_{\text{unif}}$ encode any distribution by a table, where each row is a distribution value for a point in $\Xcal^d$. Suppose the rows are well ordered so we do not need to encode the points for simplicity. Then $U_{\text{unif}}$ only needs to encode $(2^{m})^{d}$ numbers, with each number occupying $n$ bits. After coding the distribution values, $U_{\text{unif}}$ uses a Huffman coding program to turn it into a prefix code.\footnote{Notice that $U_{\text{unif}}$ does not save the Huffman code on each point in $\Xcal^d$, instead, it saves the distribution value on each point. The program that inputs a distribution and outputs a Huffman code is constant w.r.t. the distribution.} 
Therefore, for any codebook $g$, any binary codeword sequence $B(p)$ is decoded by FCMs with the same model length, i.e. same $l_{U_{\text{unif}}}(T)$ in \cref{eq:two-part-objective}. For any $(m,n,d)$, $U_{\text{unif}}$ gives a uniform prior over all codebooks that are Shannon codes of a distribution on $\Xcal^d$ with precision $(m,n)$. Using $U_{\text{unif}}$ as UFCC, the objective~\cref{eq:two-part-objective} is equivalent to maximum likelihood. Namely, from the perspective of $U_{\text{unif}}$, no codebook is simpler or more preferable than another.
\begin{figure}[ht]
    \centering
    \includegraphics[width=.9\linewidth]{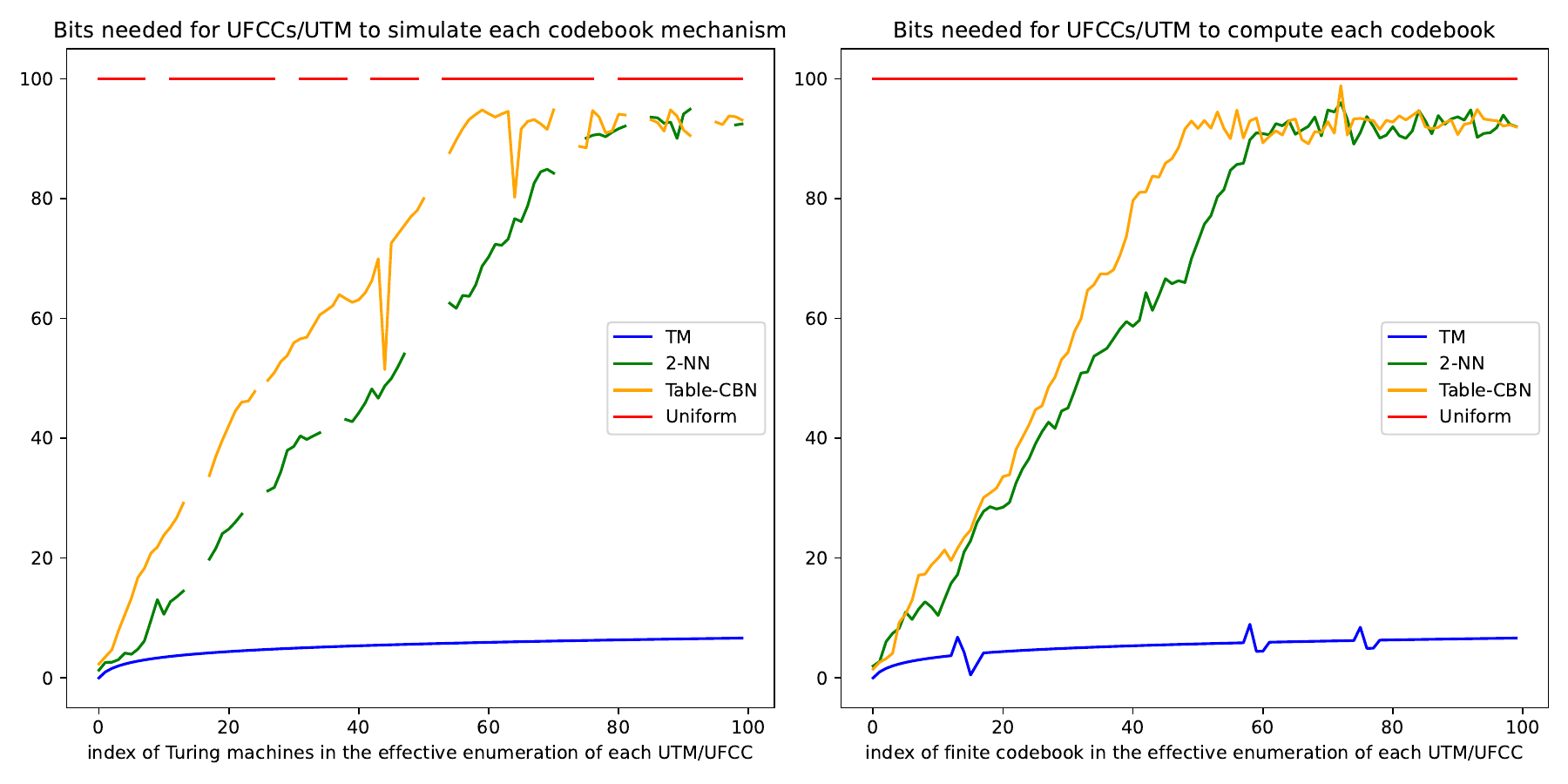}
    \caption{Given $(m,n,d)$, we consider the codebooks on $\Xcal^d:=(\Bcal^m)^d$ with precision $(m,n)$. The curves are fictitious for illustration, since some of them are not computable. Left figure: The x-axis is the index of Turing machines in an effective enumeration of all Turing machines that can compute a codebook. The y-axis is the coding length using a universal Turing machine or a UFCC. Right figure: The x-axis is the index of codebooks in an effective enumeration of all codebooks. The y-axis is the minimum coding length of the Turing machine or FCM that the universal Turing machine or UFCC can simulate. Blue line: Turing machines simulated by an arbitrary universal Turing machine. Green line: FCMs simulated by a universal two-layer neural network computer. Yellow line: FCMs simulated by $U_{\text{TabCBN}}$ defined in~\cref{prop:factorize_shorter}. Red line: FCMs simulated by $U_{\text{unif}}$ defined above.}
    \label{fig:comp_UFCC}
\end{figure}

Back to our question: are some UFCCs better than others? We conjecture that a criterion for good UFCC is: for any additively optimal UTM, a good UFCC $U$ should have a similar landscape (i.e. the order) as the UTM in the right figure of~\cref{fig:comp_UFCC}. Namely, a good UFCC should preserve the order of codebooks in a certain UTM.
\vspace{-.25cm}
\section{Case studies}\label{sec:case_studies}
We study the solutions for~\cref{eq:two-part-objective} under some particular UFCCs, showing that compression leads to selecting CFMPs that have algorithmic causal or symmetric structures. In the following, all the UFCCs simulate FCMs by composing a CFMP with a Huffman coding program, and all the CFMPs only involve no-hidden-variable mechanisms.
\vspace{-.2cm}
\subsection{Causal factorizations and sparse mechanism shifts}\label{subsec:SMS}
Consider a UFCC $U_{\text{TabCBN}}$ that first takes $m,d,n$ as input where $m$ denotes the precision of $\Xcal$, i.e., $|\Xcal|=2^m$, and $d$ denotes the number of variables, and $n$ denotes the precision of discrete distribution values. Let $U_{\text{TabCBN}}$ only allow the CFMPs in the following form: each CFMP $\alpha$ is a causal Bayesian network defined in 1. of~\cref{eg:CFMP}, with the further constraints that
\begin{enumerate}
    \item All elements in $\Pcal_\alpha$ are conditional probability tables. Each $f\in \Pcal_\alpha$ is incompressible, i.e., the Kolmogorov complexity $C(f)$ equals the coding length of its probability table.
    \item $\Phi_\alpha$ is restricted to projections on variables: for $d$ variables, there are $2^d$ possible projections, and we define a uniform code on these elements so that each element needs $d$ bits.
\end{enumerate}
\begin{restatable}{proposition}{propFactorizeShorter}\label{prop:factorize_shorter} (Causal factorization is shorter than encoding the joint distribution)
    Suppose $\PP$ is supported on $\Xcal^d \times [I]$ with precision $(m,n)$ and $\PP(X,e_i)= \PP^i(X_1 | X_{S_i}, e_i) \PP(X_2, \dots, X_d) \PP(e_i)$ (w.l.o.g. suppose $\PP(e_i)=\frac{1}{I}$) and $(S_i)_{i\in [I]}$ are subsets of $[d]$ with $|S_i|<d-1$. Suppose $I=2^{\log I}\leq 2^n$. Denote $\alpha, \beta$ as two CFMPs simulated by $U_{\text{TabCBN}}$ and computing the same $\PP$, while $\alpha$ is a CBN model (\cref{eg:CFMP} (1))  proceeds by separately saving and calling factorized mechanisms, and $\beta$ is a statistical density estimator (\cref{eg:CFMP} (6)) which proceeds by encoding the whole multi-env distribution in $\Pcal_\beta$.
    Then $l_{U_{\text{TabCBN}}} (\alpha)(m,d) = o(l_{U_{\text{TabCBN}}} (\beta)(m,d))$
    as $m\to \infty$ or $d\to \infty$.
\end{restatable}

Although $U_{\text{TabCBN}}$ can compute any finite codebook, and by \cref{cor:correpondance_proba_code}, it can equivalently compute any multi-env distribution (\cref{cor:correpondance_proba_code}), its encoding is inefficient because it encodes all elements in $\Pcal_\alpha$ in the form of probability tables. To relax this assumption, we need to assume compressibility of $\Pcal_\alpha$. Consider a UFCC $U_{\text{CompCBN}}$ which is defined as follows: for all $\alpha$,
\begin{enumerate}
    \item $\Pcal_\alpha$ is compressible and finite, of cardinal $M$: there exists a Turing machine $T$ such that for all $f\in \Pcal_\alpha$, $C_{U_{\text{CompCBN}}}(f | T)=\log M \leq l_{U_{\text{TabCBN}}}(f)$. Namely, we do not need to store each probabilistic mechanism as a table. Instead, $\alpha$ can run the subprogram $T$ in step 1, and then use a uniform code $\log M$ to encode each probabilistic mechanism. For example, we can write a Python program to generate a binomial distribution $B(n,p)$ with only two parameters, instead of storing each probabilistic map as a table.
    \item Same as in $U_{\text{TabCBN}}$, $\Phi_\alpha$ is restricted to projections on variables: for $d$ variables, there are $2^d$ possible projections, encoded uniformly.
    \item We consider two strategies for implementing step 3 of CFMPs in $U_{\text{CompCBN}}$. We assume that we want to use $N$ featurized mechanisms which are elements of a subset of size $k$ of all $M$ available featurized mechanisms.
    \begin{enumerate}
        \item Strategy 1: Directly select $N$ mechanisms from $M$ mechanisms, i.e., encode the selected mechanisms by a map $[N]\to [M]$.
        \item Strategy 2: First, write down $k$, the number of featurized mechanisms that are used. Then choose the $k$ selected mechanisms encoded by a map $[k]\to [M]$.
        Moreover, encode the final assignment of the $N$ mechanisms through a (surjective) map $[N]\to [k]$.
    \end{enumerate}
\end{enumerate}

Under this UFCC, we prove that the factorization that uses the sparsest mechanism shifts compresses best among all factorizations:
\begin{restatable}{proposition}{propSMS}\label{prop:SMS}
In $U_{\text{CompCBN}}$, for $k=o(N)$ and $k<\frac{M}{2}$, strategy 2 has a shorter coding length than strategy 1, and the difference in coding length between strategy 1 and strategy 2 is decreasing in this range of $k$.
\end{restatable}
\begin{remark}
   ~\cref{prop:SMS} shows that in the trade-off of~\cref{eq:two-part-objective},  Strategy 1 can be preferred to Strategy 2 when the data is noisy, namely when the coding length of the codeword part $-\log \PP(x)$ increases. The intuition behind this is that ``identity'' or ``sparse mechanism shifts'' are all approximations; if the model with sparse mechanism shifts fits the data sufficiently well, compression prefers not to code too many mechanisms. The experiments in \cref{sec:experiments} also show this trade-off.
\end{remark}
\begin{figure}[ht]
\centering
\includegraphics[width=0.8\linewidth]{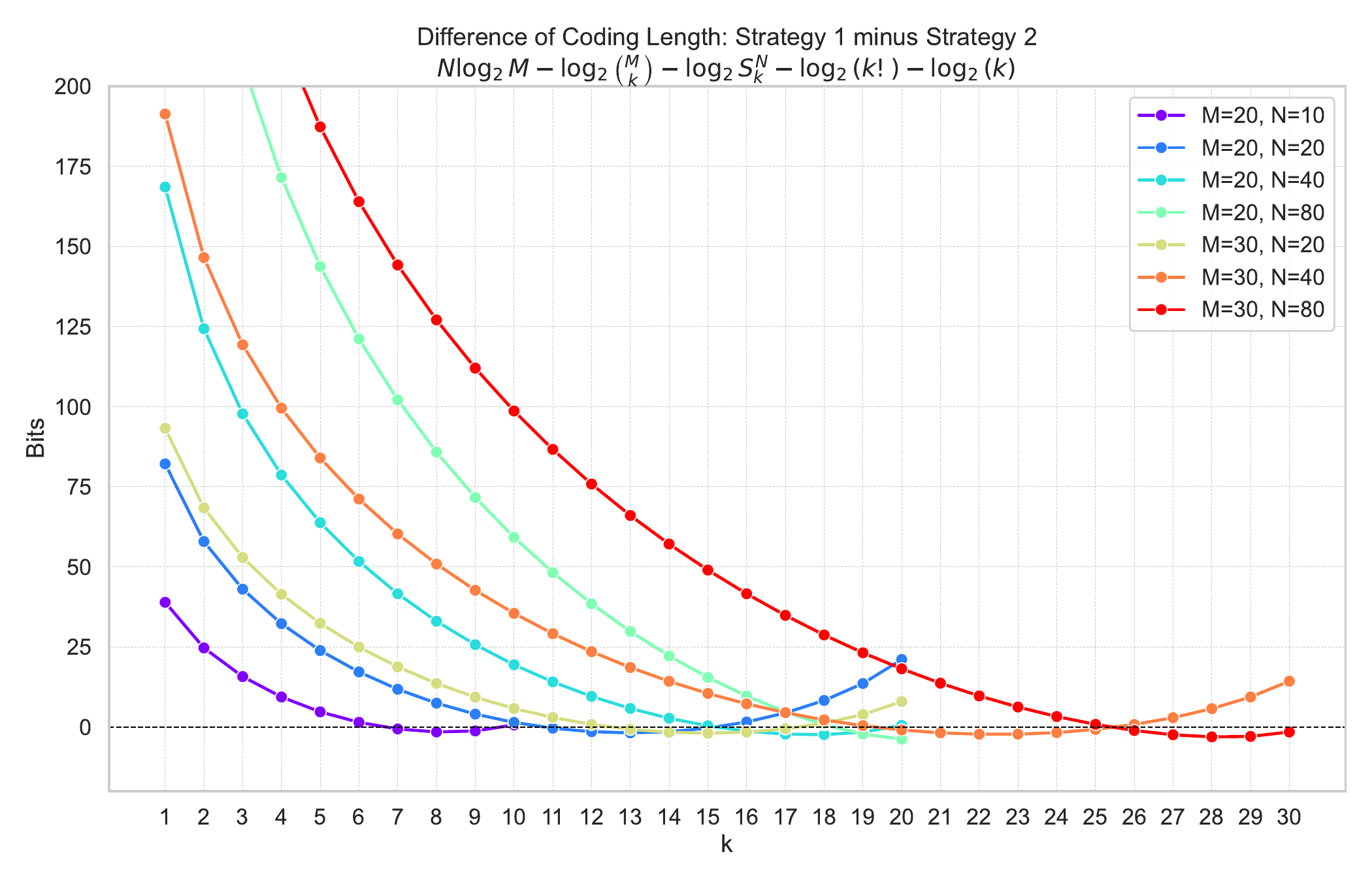}
\captionsetup{width=\textwidth}
\caption{Illustration of~\cref{prop:SMS}. Using $U_{\text{compCBN}}$, the difference in coding length of strategy 1 minus strategy 2 is initially positive and decreases with $k$, then becomes negative.
Under $U_{\text{compCBN}}$, the objective \cref{eq:two-part-objective} minimizes the sum of Shannon code length and the length of CFMP. Strategy 2 is preferred when using few featurized mechanisms is precise enough for the CFMP to model the distribution of the multi-env system.
}
\label{fig:comparison_terms}
\end{figure}
\subsection{Symmetries}
Consider a UFCC $U_{\text{TabInv}}$ that first takes $m,d,n$ as input where $m$ denotes the precision of $\Xcal$, i.e., $|\Xcal|=2^m$, and $d$ denotes the number of variables, and $n$ denotes the precision of discrete distribution values. Let $U_{\text{TabInv}}$ only allow the CFMPs in the form of (4) in~\cref{eg:CFMP} and with further constraint same as the first constraint in the definition of $U_{\text{TabCBN}}$, plus that
\begin{enumerate}
    \item For all $\alpha$ simulated by $U_{\text{TabInv}}$, $\Phi_\alpha$ is restricted to the compositions of projections on variables and quotient maps $\phi_G: \Xcal\to \Xcal / G$ of a certain group action $G$ that only acts on one dimension in $\Xcal^d$ and leaves other dimensions unchanged. Namely, any element in $\Phi_\alpha$ is in the form $\phi_G\circ \pi$.\footnote{If we allow more flexible quotient maps we can detect more symmetric structures. Those constraints are written in the UFCC, which is shared by the sender and receiver, so they do not take up bits in the FC complexity.}
    \item The selected mechanisms in Step 3 is the same for all $x\in\Xcal^d$, namely there is no context-specific mechanism (\cref{eg:CFMP} (2)).
\end{enumerate}
\begin{restatable}{proposition}{propInvShorter}\label{prop:inv_shorter} (Invariance factorization is shorter than Markov factorization)
    Suppose $\PP$ is supported on $\Xcal^2$ and $\PP(X)= \PP(X_1 | \phi(X_2)) \PP(X_2)$, where $\phi: \Xcal\to \Xcal / G$ is the quotient map of a certain group action $G$ that only acts on one dimension in $\Xcal^2$ and leaves the other dimension unchanged. Suppose the number of orbits $\Xcal / G$ is  independent of the precision $m$. Denote $\alpha, \beta$ as two CFMPs simulated by $U_{\text{TabInv}}$ and computing the same $\PP$, while $\alpha$ is a $G$-invariant learning model (\cref{eg:CFMP} (4)) which factorizes $\PP$ into $\PP(X)= f_1(\pi_1(X) | \phi\circ \pi_2(X)) f_2(\pi_2(X))$, and $\beta$ is a CBN model (\cref{eg:CFMP} (1)) which factorizes $\PP$ into $\PP(X)= f'_1(\pi_1(X) | \pi_2(X)) f_2(\pi_2(X))$ for certain $f_1, f'_1, f_2$.
    Then $l_{U_{\text{TabInv}}} (\alpha)(m) = o(l_{U_{\text{TabInv}}} (\beta)(m))$.
\end{restatable}

With a proof similar to \cref{prop:factorize_shorter} we can show that for the general case $\PP(X)= \PP(X_1|\phi\circ\pi(X))\PP(\pi(X))$, if the dimension of $\pi(\Xcal^d)$ is lower than $d-1$, then the model length of the Markov factorization $\beta$ in \cref{prop:inv_shorter} is shorter than the statistical density estimator. In this case, invariant factorization is shorter than Markov factorization, which is shorter than no factorization.

The orbits of many group actions are independent of the precision $m$, such as the reflection and translation; rotation is also the case when $m$ is sufficiently large.
\section{Experiments}\label{sec:experiments}
We illustrate our theoretical findings through simple experiments with synthetic data.

We consider two synthetic settings with sparse mechanism shifts. In both settings, the goal is to show that by minimizing FC complexity, we select a model that trades off the complexity of the model itself and the data-to-model coding length, i.e. negative log-likelihood. Then we say that according to the model selection method of minimizing FC complexity, $X$ algorithmically causes $Y$. Details of experiments are in \cref{sec:details_experiments}. The code can be found here: \href{https://github.com/WendongL/algorithmic-causality-compression}{https://github.com/WendongL/algorithmic-causality-compression}

\subsection{Covariate shifts}\label{sec:exp_cov_shifts}
Consider a multi-env system with covariate shift: $\PP(X,Y,E)=\PP(X|E)\PP(Y|X)\PP(E)$. Suppose many CFMPs $(\alpha_l)_{l\in[L]}$ generate the same $\Pcal$ and $\Phi$ and featurize them in the same way, as described in strategy 2 before \cref{prop:SMS}. The question is: among these CFMPs, under the UFCC $U_{\text{CompCBN}}$ in \cref{prop:SMS}, given multi-env finite data, if we only consider the case ``$X$ causes $Y$'', which CFMP should we select?
\begin{figure}[ht]
    \centering
    \includegraphics[width=0.45\linewidth]{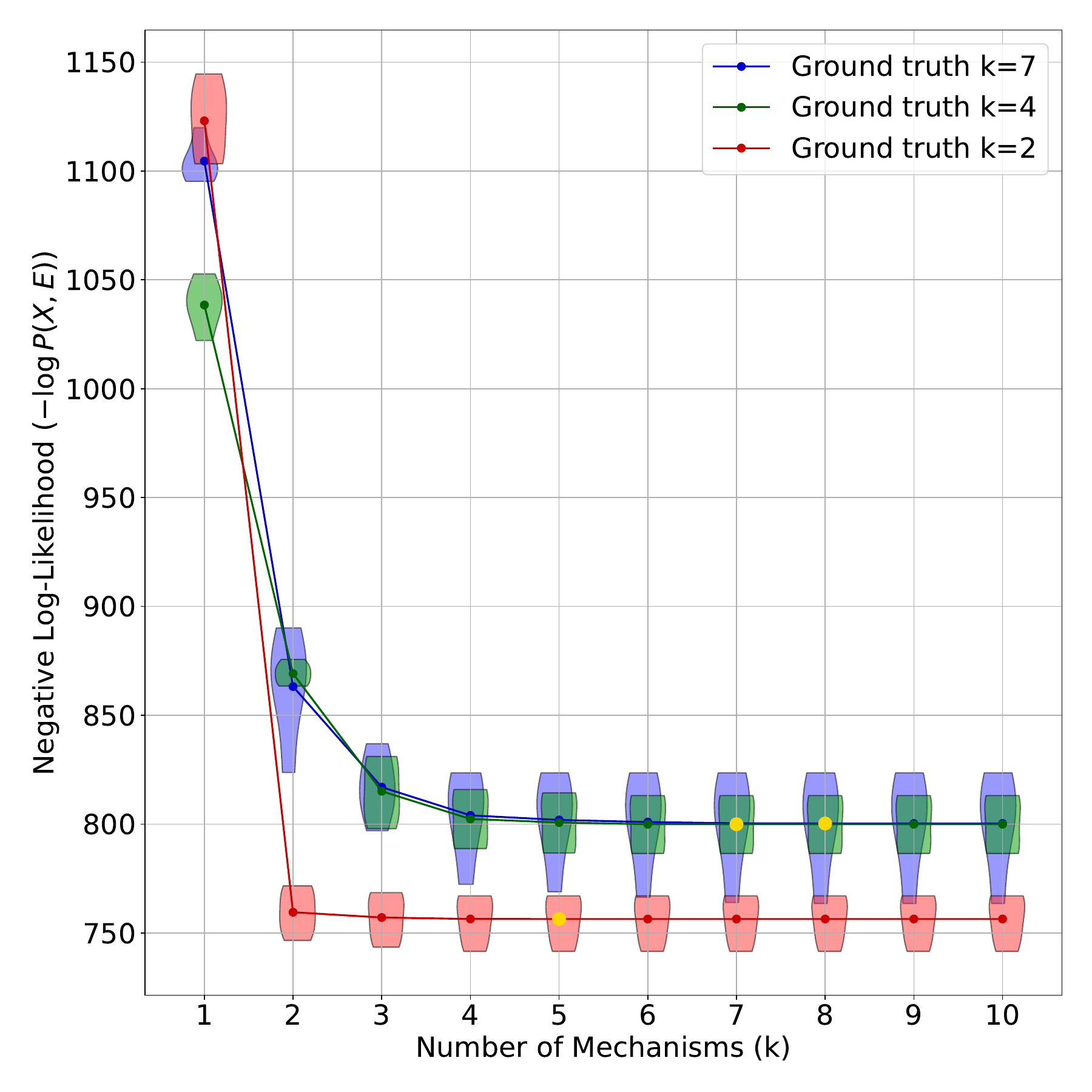}\hspace{0.08\linewidth}
    \includegraphics[width=0.45\linewidth]{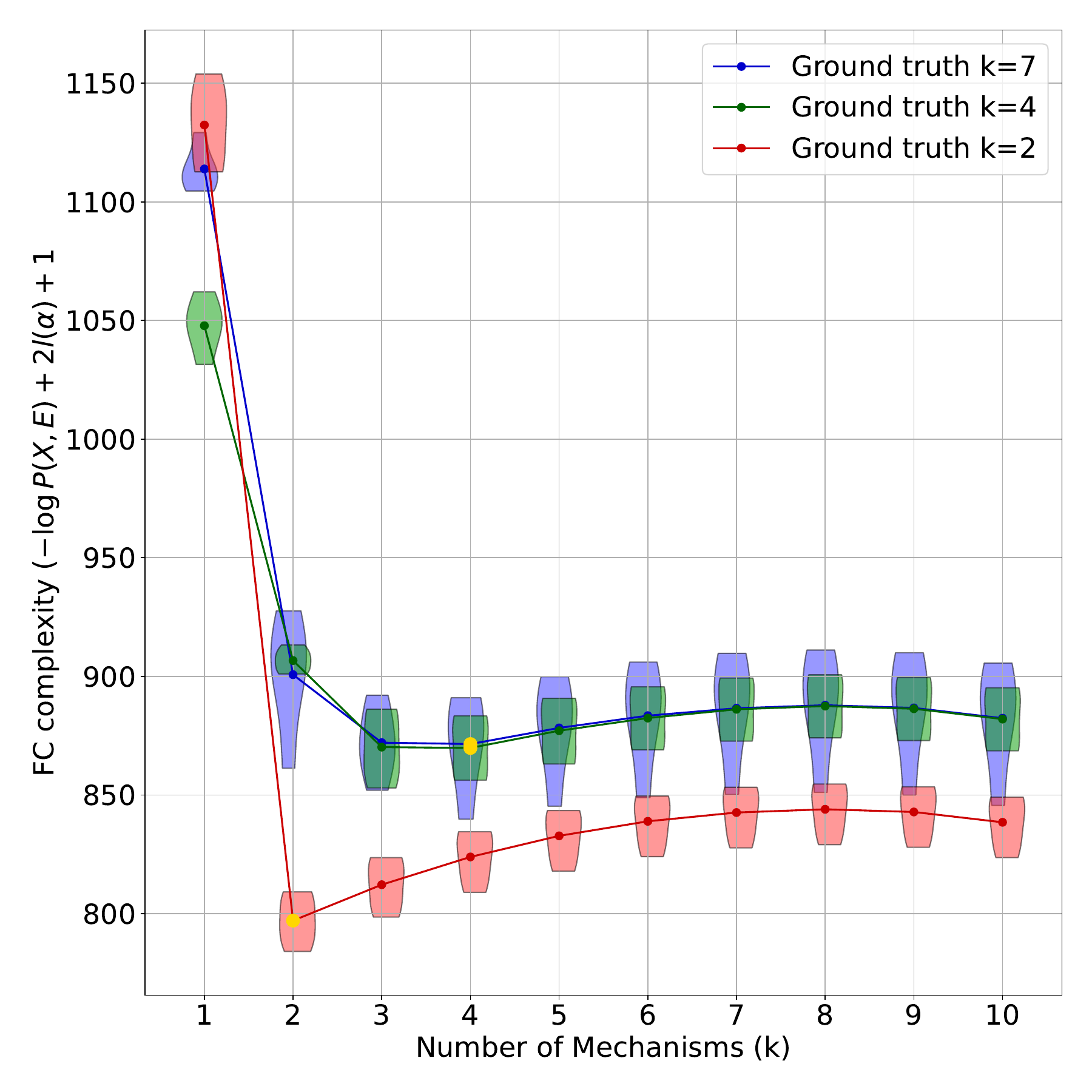}
    \caption{Results in \cref{sec:exp_cov_shifts}. The left figure shows the minimal negative log-likelihood of the CFMPs that use $k$ mechanisms $\PP(X|E)$. The right figure shows the minimal FC complexity (NLL+model coding length $2l_{U_{\text{CompCBN}}}(\alpha)+1$ (\cref{eq:two-part-objective})) of the CFMPs that use $k$ mechanisms $\PP(X|E)$. We choose 3 different multi-env distributions to generate the data, respectively using 2,4 and 7 mechanisms among 10 environments. The experiments are run with 5 seeds. The argmin $k$ are highlighted.}
    \label{fig:exp_covariate_shift}
\end{figure}
\begin{figure}[ht]
    \centering
    \includegraphics[width=0.5\linewidth]{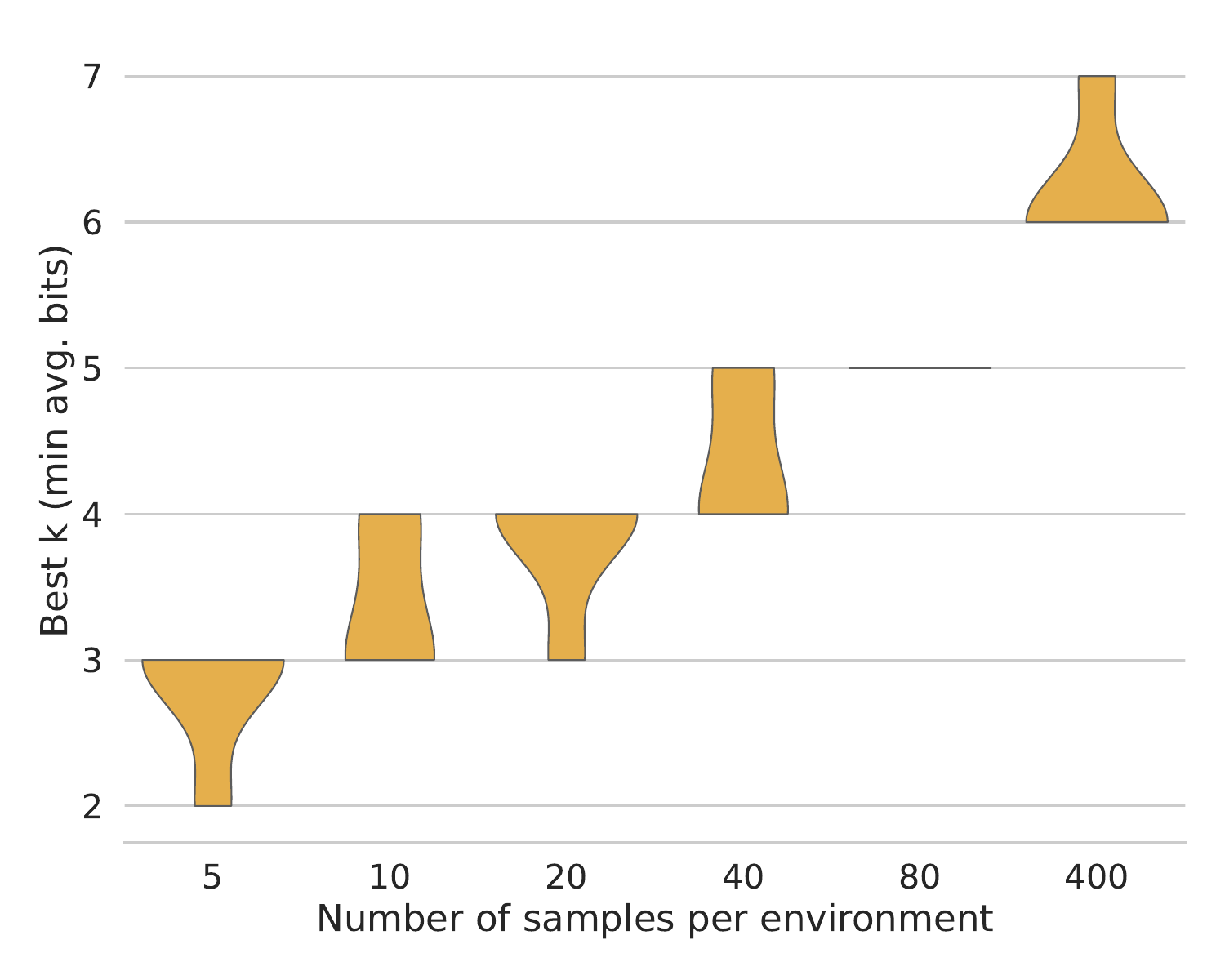}
    \caption{For the experiment with ground truth $k=7$ for 10 envs, we increase the number of samples per env. As the number increases, the model selected by minimizing FC complexity tends to use more mechanisms in different environments.}
    \label{fig:best_k_violin_with_penalty}
\end{figure}
In the current experiment, we generate $10$ environments for $(X,Y)$ supported in $\{0,1,\dots 17\}^2$, each having $10$ iid samples. We make the shifts in $\PP(X|E)$ sparse, i.e. less than $10$, see the legends in \cref{fig:exp_covariate_shift}. $\PP(X|E)$ is modeled by Poisson distribution, with the support outside of $\{0,1,\dots 17\}^2$ absorbed in $(X=17)$. We define the invariant $P(Y|X)\sim \Ncal(X, 1)$.

To proceed with model selection, we consider a set of CFMPs $(\alpha_l)_{l\in[L]}$ which have $18$ featurized mechanisms for $\PP(X|E)$, and the known $\PP(Y|X)$ and $\PP(E)$. Since the goal is to compare and select among these CFMPs, we do not calculate the coding length for steps 1 and 2 which is the same for all those CFMPs. The only step that makes a difference in coding length, step 3, costs $\log\binom{M}{k} + \log S_k^N + \log (k!) + \log k$ bits, as shown in \cref{fig:comparison_terms} and the proof of \cref{prop:SMS}. Namely, the coding length in step 3 depends on how many different $\PP(X|E)$ mechanisms $\alpha$ should use. Once $k$ is chosen, $\alpha$ performs a greedy search over all possible choices of using $k$ mechanisms in $10$ environments and selects the choice that minimizes the negative log-likelihood (NLL) of multi-env data, which is equivalent to the sum of Shannon code length for the multi-env data.

The results in \cref{fig:exp_covariate_shift} show that with finite data, if we only minimize NLL, we tend to select the ground truth mechanisms, and sometimes we choose even more mechanisms than the ground truth needs. However, by minimizing FC complexity (NLL+model length), we tend to trade off the NLL and the model complexity and select a simple (sparse mechanisms needed) model that explains the data well enough.

\subsection{Causal discovery without identifiability}\label{sec:exp_cd}
\begin{figure}[ht]
    \centering
    \includegraphics[width=0.45\linewidth]{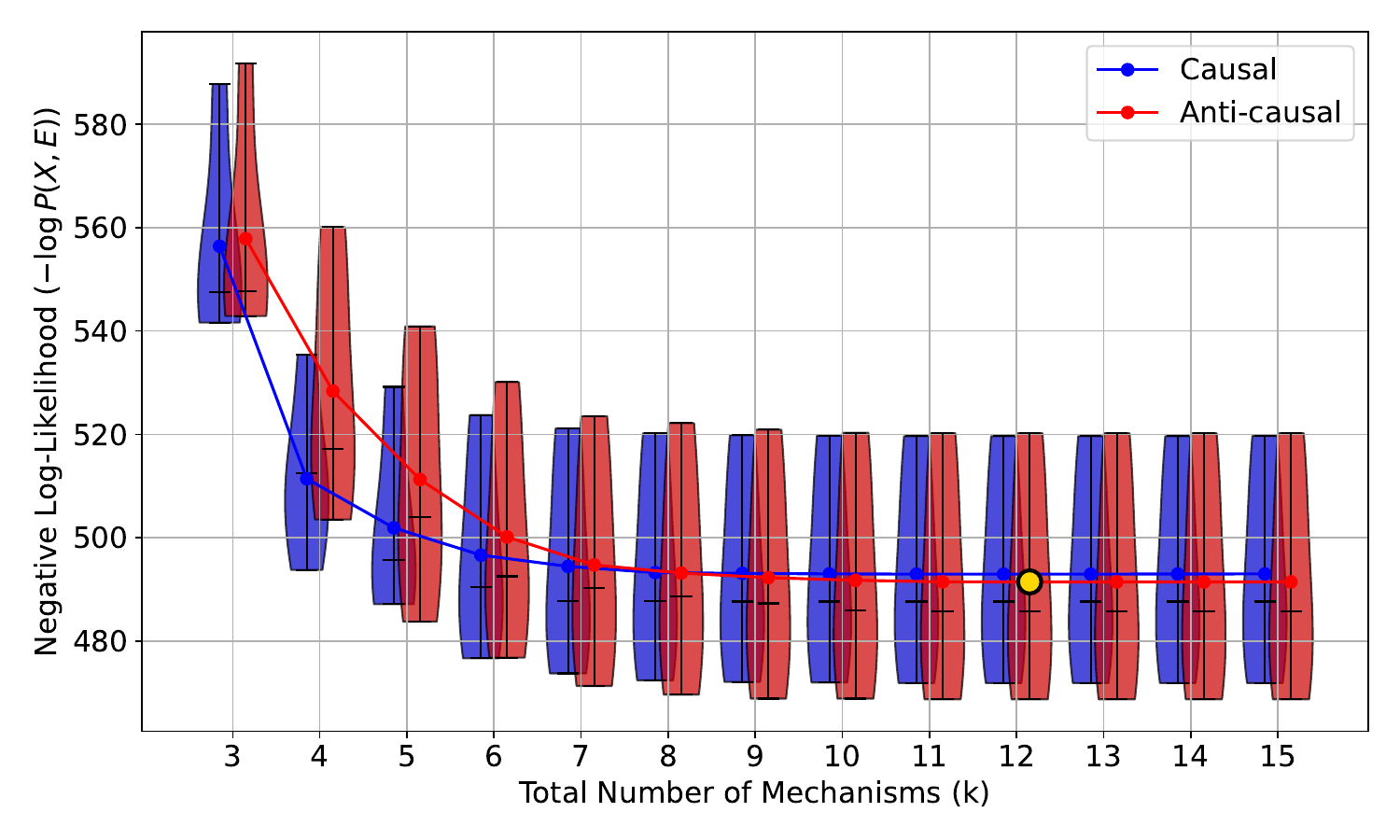}\hspace{0.08\linewidth}
    \includegraphics[width=0.45\linewidth]{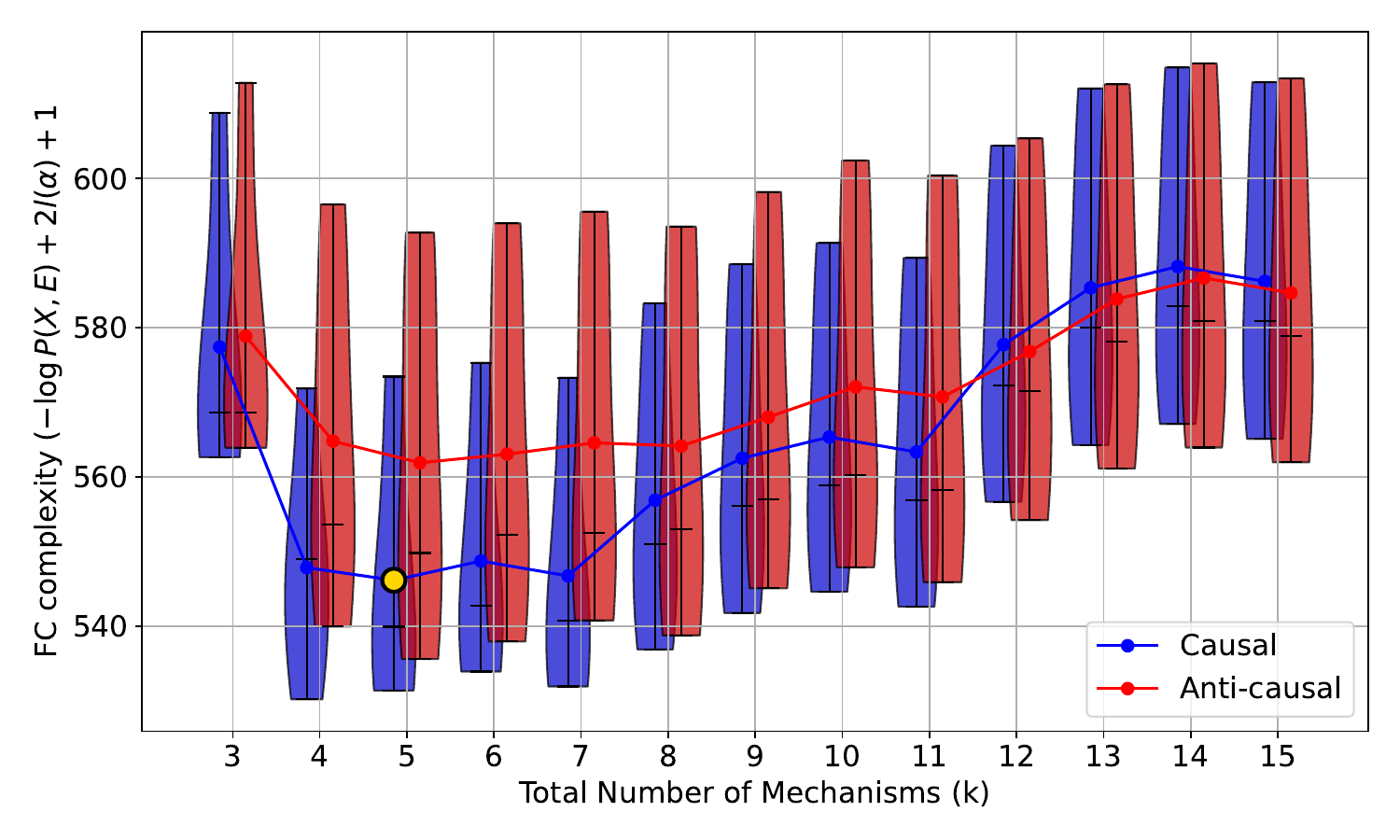}
    \caption{Results in \cref{sec:exp_cd}. The x-axis is always the overall number of mechanisms (maximum 24 but truncated at 15). The left figure shows the minimal negative log-likelihood of the CFMPs that use $k$ mechanisms. The right figure shows the minimal FC complexity (NLL+model coding length $2l_{U_{\text{CompCBN}}}(\alpha)+1$ (\cref{eq:two-part-objective})) of the CFMPs that use $k$ mechanisms. The experiments are run with 5 seeds. The argmin $k$ are highlighted.}
    \label{fig:exp_cd}
\end{figure}
Consider a multi-env system with 5 environments generated by linear Gaussian SCMs:
$X\sim \mathcal{N}(0,\sigma_1^2), Y=a X + \epsilon$ with $\epsilon\sim \mathcal{N}(0,\sigma_2^2)$.
We generate 10 iid samples for each environment.

In total, 8 parameters are needed to fit the data optimally.

However, the causal graph is not identifiable, because we can also find parameters for each env using the linear Gaussian model $Y\to X$:
$Y\sim \mathcal{N}(0,\tau_1^2), X=b Y + \epsilon$ with $\epsilon\sim \mathcal{N}(0,\tau_2^2)$.

For both causal and anti-causal linear Gaussian models, we allow 8 choices for each parameter, which include the optimal parameters, in total 24 parameters for the multi-env system. We constrain the number of mechanisms (in this case, parameters) before training the model, just as Strategy 2 in Proposition 18.

We see that FC complexity provides a new criterion of a “good model” that is applicable when the model is not identifiable from infinite sample and addresses the model selection problem, e.g., here fewer mechanisms than in the ground truth are selected.

\section{Related work}\label{sec:related_work}
\citet{janzing2010causal} were the first to consider causal mechanisms implemented by Turing machines. They replace statistical (conditional) independence with algorithmic (conditional) independence, conjecturing that if the Kolmogorov complexity of a string or a joint distribution can be decomposed into a sum of conditional Kolmogorov complexities of the causal mechanisms according to a graph, then this graph should be selected. They also extended their model selection principle to probabilistic models \citep[Postulate 7]{janzing2010causal}, which we comment in detail in \cref{sec:remark_AMC} and compare with our \cref{priniple:compression}.
The incomputable objective in \cite[Postulate 7]{janzing2010causal} is replaced by entropy~\citep{SteJanSch10,pranay2021causal} or MDL~\citep{Budhathoki,budhathoki2017mdl,marx2019causal,mian2021discovering,mian2023information} in all subsequent papers. We discuss in~\cref{sec:comment_entropy} and~\cref{sec:MDL} the difference between our approach and them.
\cite{marx2021formally} discusses the relationship between Postulate 7 and two-part code. We give our comments on their results and on Postulate 7 in \cref{sec:remark_AMC}.
Some papers~\citep{marx2019identifiability, marx2019telling, mameche2022discovering, mameche2024learning} claim identifiability (i.e., recovery of the ground-truth graph) by minimizing their proxy of Postulate 7 in \cite{janzing2010causal}. We show in~\cref{lem:IdentminCE} that any identifiability of graphs is reduced to the claim of the uniqueness of the solution of minimum cross-entropy instead of minimizing a bound of Kolmogorov complexity. Our approach focuses on an upper bound of Kolmogorov complexity of a specific class of Turing machines that compute probabilistic models. Our objective is fundamentally different from \cite[Postulate 7]{janzing2010causal}, as we show in \cref{sec:remark_AMC}.

\cite{dhirbivariate} address bivariate causal discovery without confounding by comparing the posteriors of two graphs, with the correctness (probability of selecting the ground truth graph) depending on the total variation between the ground-truth distribution and marginal likelihood. In \cref{sec:MDL}, we discuss the difference between our approach and Bayesian model selection.

On the side of computation theory, there is abundant work on constraining the definition of Kolmogorov complexity to make it computable, such as 
resource-bounded complexity~\citep{barzdin1968complexity}, logical depth~\citep{chaitin1977algorithmic}, automatic complexity~\citep{shallit2001automatic}.%
Those definitions are fit for compressing more general strings without any probabilistic structure, therefore the entropic code ($-\log \PP$) is not applicable. 
In the domain of knowledge representation, \cite{shen2018conditional,kisa2014probabilistic} use probabilistic sentential decision diagrams (PSDD) to model Bayesian networks and learn them by maximum likelihood. Some use instead trees~\citep{chen2022definition} and arithmetic circuits~\citep{darwiche2022causal, huang2024causal}.

\section{Discussion}

If it is the case that compression in some cases may {\em automatically} yield causal structure, then this has significant implications for modern machine learning. For instance, there is an ongoing debate as to whether large language models (LLMs) can
understand causality in the sense of correctly applying causal principles across a range of problems~\citep{Jinetal23}. After all, a large extent of apparent causal knowledge may be explained by simply regurgitating causal knowledge which is abundant in the training set, and one may thus argue that the apparent causal knowledge of LLMs may be entirely superficial.
The arguments put forward in the present paper suggest that perhaps these two extremes may not be entirely irreconcilable: training a (large, but finite) model on a significant fraction of the internet necessarily forces a model to compress data, and even though the classical identifiability assumptions do not hold, an algorithmic causal model can still be extracted. In other words, if the empirical scaling laws for LLMs continue to hold, the model may have no choice but to learn (algorithmic) causality. Similar indirect arguments have been made for non-causal learning for the case of grokking~\citep{power2022grokking} and the emergence of complex skills~\citep{arora2023theoryemergencecomplexskills}.
An open question is whether LLMs simulate nontrivial CFMPs, i.e., whether they call and reuse some mechanisms like in CFMP. We know that in principle, they can~\citep{perez2021attention}. 

The use cases of algorithmic causality and Pearl's causality are disjoint. Algorithmic causality is not a competitor of Pearl's causality, since (i) if we have infinite data and the likelihood converges, compression is trivialized to minimum cross-entropy, which outputs any Turing machine that computes the same probability distribution, so the model selection by any UFCC suffers from the same non-identifiability as Pearl's framework; (ii) algorithmic causality deals with multi-env but correlational data and is a rung 1 model in Pearl's hierarchy; (iii) its advantage lies only in those cases where the intervention targets are not certain, for example in LLM pre-training, there is no prior knowledge about which context token is the environment label or intervention-target variable. In this case, using Pearl's causality might not be as appropriate as using algorithmic causality; on the other hand, if scientists have data generated by some strictly randomized controlled experiments on many variables respectively, then the identifiability results in Pearl's causal models are more convincing than algorithmic causality for the scientific community.

\acks{The authors thank Frederik Hytting Jørgensen and the anonymous reviewers for helpful comments and discussions. This work was supported by the T\"ubingen AI Center.}

\bibliography{ref}

\newpage
\appendix
\begin{center}
{\centering \LARGE APPENDIX}
\vspace{0.8cm}
\sloppy
\end{center}
\input{App_overview}
\input{App_notations}
\input{App_preliminaries}

\section{Proof in~\cref{sec:prelim_ident}}\label{sec:proof_sec_prelim_ident}
\lemIdentminCE*
\begin{proof}
    (i) First, consider the observational CBN model class. Denote $\theta^*$ as the parameter of the distribution of which the marginal distribution generates $\Dcal_n$, it is a solution of
    \begin{equation}
        \argmax_{(G,\theta)\in \Mcal} \lim_{n\to\infty} \frac{1}{n}\log L(\theta|\Dcal_n).
    \end{equation}
    For any $\theta \neq \theta^*$, by~\cref{def:identCDCRL}, $\PP_\theta$ and $\PP_{\theta}^*$ are different on a $\mu$-non-negligible set. So the random variable $\frac{p_{\theta}(X)}{p_{\theta^*}(X)}$ is non-degenerate. In addition, the log function is strictly convex, so by Jensen's inequality,
    \begin{equation}
      \EE_{\PP_{\theta^*}} \left[\log \frac{p_{\theta}(X)}{p_{\theta^*}(X)} \right] < \log \EE_{\PP_{\theta^*}} \left[\frac{p_{\theta}(X)}{p_{\theta^*}(X)} \right].
    \end{equation}
    The right-hand side is zero because
    \begin{equation}
      \EE_{\PP_{\theta^*}} \left[\frac{p_{\theta}(X)}{p_{\theta^*}(X)} \right] = \int_\Xcal \frac{p_{\theta}(x)}{p_{\theta^*}(x)} p_{\theta^*}(x) dx = 1.
    \end{equation}
    Therefore 
    \begin{equation}\label{eq:lem_identminCE}
        \EE_{\PP_{\theta^*}} \left[\log p_{\theta}(X) \right] < \EE_{\PP_{\theta^*}} \left[\log p_{\theta^*}(X) \right].
    \end{equation}
    
    Equivalently, by applying the law of large numbers on \cref{eq:lem_identminCE},
    \begin{equation}
        \lim_{n\to\infty} \frac{1}{n}\log L(\theta|\Dcal_n) <  \lim_{n\to\infty} \frac{1}{n}\log L(\theta^*|\Dcal_n).
    \end{equation}

    Therefore, maximum likelihood or minimum cross-entropy has a unique solution $\theta^*$. Since we assume that $\Mcal$ is identifiable, which means that each $\theta$ only has one underlying graph $G$, we conclude that there is only one solution $(G^*,\theta^*)$.

    (ii) The proof for the multi-env CBN model class is the same, but with different sample space: $\Xcal^{dI}$. It is a $I$-fold copy space for $\Xcal^d$, with distributions from different environments lying in different dimensions. $X_i^j$ is the $i$-th dimension in the original sample space $\Xcal^d$ in the $j$-th environment. We use $\overline{\theta}=(\theta_1,\dots \theta_I)$ to denote the parameter for a distribution in $\Xcal^{dI}$, which uniquely corresponds to the multi-env distributions $(\PP_{\theta_i})_{i\in [I]}$. We use $\overline{\Theta}$ to denote the parametric family of all $\overline{\theta}$ that is generated from the model class $\Mcal$. By assumption of interventional environments, in $\Mcal$, $\PP^i \indep \PP^j$ for all $i\neq j \in [I]$, any probability distribution in the parametric family $\overline{\Theta}$ can be written as follows:
    \begin{equation}
        \PP_{\overline{\theta}} (X)= \prod_{i=1}^I \PP_{\theta_i} (X^i)
    \end{equation}

    Notice that on the left, $X$ is a random variable in $\Xcal^{dI}$.

    The sum of cross-entropy 
    \begin{equation}\label{eq:lem_CE_1}
        \sum_{i=1}^I \EE[-\log \PP_{\theta_i}(X)]= \EE[-\log \PP_{\overline{\theta}}(X)]
    \end{equation}
    which is exactly the cross-entropy of the $I$-fold variable in $\Xcal^{dI}$. Same as the proof (i), if there exists $i\in [I]$ such that $\theta_i \neq \theta^*_i$, then $\PP_{\theta_i}$ and $\PP_{\theta_i}^*$ are different on a $\mu$-non-negligible set. Using the same argument of Jensen's inequality over $\overline{\theta}$, we infer that $\EE_{\PP_{\overline{\theta}^*}} \left[\log p_{\overline{\theta}}(X) \right] < \EE_{\PP_{\overline{\theta}^*}} \left[\log p_{\overline{\theta}^*}(X) \right]$. By \cref{eq:lem_CE_1}, this is equivalent to 
    \begin{equation}
        \sum_{i=1}^I \EE[-\log \PP_{\theta^*_i}(X)] < \sum_{i=1}^I \EE[-\log \PP_{\theta_i}(X)]
    \end{equation}
    Therefore, maximum likelihood or minimum cross-entropy has a unique solution $(\theta^*_i)_{i\in [I]}$. Since we assume that $\Mcal$ is identifiable, which means that each tuple $(\theta^*_i)_{i\in [I]}$ only has one underlying tuple of graphs $(G_i)_{i\in [I]}$, we conclude that there is only one solution $((G_i)_{i\in [I]},(\theta^*_i)_{i\in [I]})$.

\end{proof}

\input{App_UFCC}

\input{App_case_studies}

\input{App_experiments}

\input{App_AMC}

\input{App_MDL}
\end{document}

%% file: App_overview.tex
\section*{Overview}
\begin{itemize}
    \item Appendix~\cref{sec:notations} recapitulates the notation used in this paper.
    \item Appendix~\cref{sec:Preliminaries} contains preliminary backgrounds.
    \begin{itemize}
        \item Appendix~\cref{sec:Computation} contains preliminaries in computation theory.
        \item Appendix~\cref{sec:Discrete_proba} contains preliminaries in discrete probability theory.
        \item Appendix~\cref{sec:comment_entropy} contains preliminaries in compression (source coding) in information theory.
    \end{itemize}
    \item Appendix~\cref{sec:proof_sec_prelim_ident} contains the proof in~\cref{sec:prelim_ident}.

    \item Appendix~\cref{sec:proof_undecidable} contains the proof in~\cref{sec:learn_compression}.
    \item Appendix~\cref{sec:supp_case_studies} contains supplemental results and proofs in~\cref{sec:case_studies}.
    \item Appendix~\cref{sec:details_experiments} contains details of experiments in \cref{sec:experiments}.
    \item Appendix~\cref{sec:remark_AMC} contains remarks on the Algorithmic Markov Condition \citep{janzing2010causal} and the subsequent works based on it.
    \item Appendix~\cref{sec:MDL} contains remarks on the Minimum Description Length (MDL) principle and Bayesian model selection, their difference and relation to our approach.
\end{itemize}

%% file: App_notations.tex
\newpage
\section{Notations}\label{sec:notations}
\begin{table}[h]
\centering
\begin{tabular}{ll}
\hline
Symbol & Description \\
\hline
$\log$ & Simplified symbol for $\log_2$\\
$\QQ$ & The set of rational numbers\\
$G$ & A directed acyclic graph
with nodes $V=[d]$ and arrows $E$ \\
$[d]$ & The natural numbers $1, \ldots, d$ \\
$\pa^G(i)$ & Parents of $i$, defined as $\{j\in V(G)\mid (j,i)\in E(G)\}$ \\
$\Bcal$ & Binary alphabet $\{0,1\}$\\
$\Xcal$ & The discrete finite space for one dimension of samples, of cardinal $2^m$\\
$\Omega$ & The discrete finite space for one dimension of the formal variable $\omega$, of cardinal $2^m$\\
$\epsilon$ & Empty string\\
$\alpha$ & A conditional feature-mechanism program (CFMP) (\cref{def:cond_feat_mechanism_program})\\
$\Gamma_x$ & Cylinder, defined in~\cref{def:cylinder}\\
$\Pcal_\alpha$ & List of probabilistic mechanisms (\cref{def:prob_mech_map}) generated by CFMP $\alpha$\\
$\Phi_\alpha$ & Set of feature mechanisms (\cref{def:feature_mechanism}) generated by CFMP $\alpha$\\
$f_1$ & A probabilistic mechanism (\cref{def:prob_mech_map}) \\
$g.x$ & Group action on $x\in\Xcal^d$, with $g\in G$ for a certain group $G$\\
$C_U(x)$ & Kolmogorov complexity of a string $x$ under a universal Turing machine $U$ (\cref{def:KC})\\
$\Bar{n}$ & Self-delimiting code of a natural number $n$, see~\cref{def:self_delim_code}\\
$\langle x,y \rangle$ & Self-delimiting concatenation of natural numbers $x, y$, defined in~\cref{def:self_delim_code}\\
$l_U(T)$ & Length of the index of Turing machine $T$ in the effective enumeration of\\
 & Turing machines in a universal Turing machine or UFCC $U$, see~\cref{def:KC}\\
$l(n)$ & Length of a binary string or equivalently a natural number $n$ in the binary\\
& representation, defined in~\cref{eq:def_length_number}\\
$U_{\text{TabCBN}}$ & A UFCC that is defined and used in~\cref{sec:case_studies}. Same for $U_{\text{CompCBN}}, U_{\text{TabInv}}$\\
\hline
\end{tabular}
\end{table}

%% file: App_preliminaries.tex
\section{Preliminaries}\label{sec:Preliminaries}
\subsection{Computation theory}\label{sec:Computation}
We introduce some notions that we mentioned in the paper. We follow \cite{li2019introduction}.

We use an alphabet of binary symbols $\mathcal{B}=\{0,1\}$. The set of all finite strings over $\mathcal{B}$ is denoted by $\mathcal{B}^*$, defined as
\begin{equation}
    \mathcal{B}^*=\{\epsilon, 0,1,00,01,10,11,000, \ldots\}
\end{equation}
with $\epsilon$ denoting the empty string, with no letters. Concatenation is a binary operation on the elements of $\mathcal{B}^*$ that associates $x y$ with each ordered pair of elements $(x, y)$ in the Cartesian product $\mathcal{B}^* \times \mathcal{B}^*$. 

We now consider a correspondence of finite binary strings and natural numbers. The standard binary representation has the disadvantage that either some strings do not represent a natural number, or each natural number is represented by more than one string. For example, either 010 does not represent 2, or both 010 and 10 represent 2. We can map $\Bcal^*$ one-to-one onto the natural numbers by associating each string with its index in the length-increasing lexicographic ordering
\begin{equation}\label{eq:lexico_ordering}
    (\epsilon, 0), (0, 1), (1, 2), (00, 3), (01, 4), (10, 5), (11, 6), \dots
\end{equation}
\begin{definition}\label{def:binary_lexico_code}
    We call the \emdef{binary lexicographic code $B$} the map that maps~\cref{eq:lexico_ordering} reversely, i.e. $0\mapsto \epsilon$, $1\mapsto 0$, $2\mapsto 1$, $3\mapsto 00 \dots$.
\end{definition}

The length of a finite binary string $x$ is the number of bits it contains and is denoted by $l(x)$. One can verify that for a natural number $n$ represented by~\cref{eq:lexico_ordering}, 
\begin{equation}\label{eq:def_length_number}
    l(n)=\lfloor \log_2 (n + 1) \rfloor
\end{equation}
For the readability of the paper, we define $l(n)=\log_2 n$ instead.
\begin{definition}\label{def:self_delim_code}
    A \emdef{self-delimiting code} of a natural number $n \in \NN$ is a function $\NN\to \Bcal^*$ that maps $n\in\NN$ to
    \begin{equation}
        \Bar{n}=\underbrace{11\dots1}_{\begin{array}{c}l(n)\text{ times}\\\end{array}}\quad 0 B(n).
    \end{equation}
    where $B$ is binary lexicographic code defined in~\cref{def:binary_lexico_code}.
    
    We call $\langle \cdot,\cdot \rangle: \NN \times \NN\to \Bcal^*$ a \emdef{concatenation function}: given $m,n\in \NN$,
    \begin{equation}
        \langle m,n \rangle := \Bar{m} B(n)
    \end{equation}
    We call $\langle m,n \rangle$ a \emdef{self-delimiting concatenation} of $m,n$.
\end{definition}
The usage of the self-delimiting concatenation is in a universal Turing machine or a universal finite codebook computer (UFCC), where the input is two natural numbers. We have to use a self-delimiting concatenation so that the string $mn$ is uniquely identified as $(m,n)$.

\begin{definition}
A \emdef{Turing machine (TM)} consists of a finite program, called the \emdef{finite control}, capable of manipulating a linear list of cells, called the \emdef{tape}, using one access pointer, called the head. We refer to the two directions on the tape as right and left. The finite control can be in any one of a finite set of states $Q$, and each tape cell can contain a $0$, a $1$, or a blank $B$.
The TM can perform the following basic operations:
\begin{enumerate}
    \item  it can write an element from $A = \{0, 1, B\}$ in the cell it scans; and
    \item it can shift the head one cell left or right.
\end{enumerate}
After each step, the finite control takes on a state from $Q$, and then decides an action according to a global list of \emdef{rules}. The rules have format $(p, s, a, q)$: $p$ is the current state of the finite control; $s$ is the symbol under scan; $a$ is the next operation to be executed of type 1 or 2 designated in the obvious sense by an element from $S = \{0, 1, B, L, R\}$; and $q$ is the state of the finite control to be entered at the end of this step. If the TM enters a state that does not appear as an entry $p$ in the list of rules, then the TM \textbf{halts}. In particular, we define an extra state ``reject'' such that TM halts on ``reject''. We say that a TM \emdef{rejects} $x$ if $T$ inputs $x$ and halts on ``reject''.
\end{definition}
\begin{definition}[Formal version of~\cref{def:compute}]\label{def:compute_formal}
    A Turing machine $T$ is said to \emdef{compute} a function $f: A\subset \Xcal \to \Bcal^*$, if $T$ rejects any input outside $A$ and for all $x\in A$, $T$ halts in a finite number of steps with $f(x)$ written on the tape.
\end{definition}
Intuitively, a TM $T$ computes a function $f$ if they have the same domain $A$, and for all $x\in A$, $T(x)=f(x)$.
\begin{definition}
    We define the \emdef{equivalence relation} $\sim$ between two Turing machines: $S\sim T$ if $S$ and $T$ halt on the same inputs (i.e. considered as functions, they have the same domain $L$) and for all $x\in L$, $S(x)=T(x)$. Namely, $S$ and $T$ compute the same function.
\end{definition}
\begin{definition}[informal]\citep{li2019introduction}
    A universal Turing machine $U$ is a Turing machine that can imitate the behavior of any other Turing machine $T$.
\end{definition}
In this paper, we mean ``simulate'' by ``imitating the behavior of another Turing machine''. We have not found any formal definition of ``simulation''. For the formalization of how a universal Turing machine simulates any other Turing machines, see \cite{hennie1966two}.

\subsection{Discrete probability theory}\label{sec:Discrete_proba}
\begin{definition}\label{def:cylinder}\citep{li2019introduction}
    Let $\Bcal$ be a finite or countably infinite set of symbols. In this paper, we use $\Bcal=\{0,1\}$. A \emdef{ cylinder} is a set $\Gamma_x\subseteq \Bcal^\infty$ defined by
    \begin{equation}
        \Gamma_x=\{x\omega:\omega\in\mathcal{B}^\infty\}
    \end{equation}
    with $x\in\mathcal{B}^*$. Let $\mathcal{G}=\{\Gamma_x:x\in\mathcal{B}^*\}$ be the set of all cylinders in $\Bcal^\infty$.
A function $\mu:\mathcal{G}\to\RR$ defines a probability measure if
$$\begin{aligned}&\mu(\Gamma_\epsilon)=1,\\&\mu(\Gamma_x)=\sum_{b\in\mathcal{B}}\mu(\Gamma_{xb}).\end{aligned}$$
\end{definition}
Consider the function $\mu^{\prime}:\mathcal{B}^*\to\RR$ defined by $\mu^{\prime}(x)=\mu(\Gamma_x)$. Trivially from $\mu^\prime$ we can reconstruct $\mu$ and vice versa. From now on we identify $\mu^{\prime}$ with $\mu.$ Formally, we use the definition of measure below. One should keep in mind that our notation is shorthand for the original measure. 

\begin{definition}
    A \emdef{ finite cylinder of depth $m$} is defined by
    \begin{equation}
        \Gamma_x^m=\{x\omega:\omega\in\mathcal{B}^m\}.
    \end{equation}
\end{definition}
The probability space $\Xcal$ we consider in~\cref{def:discrete_proba_space} can be redefined as $\Xcal:=\{\Gamma_x:x\in\mathcal{B}^*\}$.

\begin{definition}\citep{li2019introduction}
    A function $\mu:\mathcal{B}^*\to\RR$ is a \emdef{ probability measure} (measure for short) if
\begin{equation}
    \mu(\epsilon)=1 \quad \text{and} \quad \mu(x)=\sum_{b\in\mathcal{B}}\mu(xb),
\end{equation}
for all $x\in\mathcal{B}^*.$ A \emdef{semi-measure} is a defective measure. A function $\mu:$
$\mathcal{B}^*\to\RR$ is a semi-measure if for all $x\in\mathcal{B}^*$,
$$\begin{aligned}\mu(\epsilon)&\leq1,\\\mu(x)&\geq\sum_{b\in\mathcal{B}}\mu(xb).\end{aligned}$$
\end{definition}

\begin{lemma}
    When the precision of $Z$ is fixed, the conditional independence set 
    $$\{X \indep Y | Z ; \quad X,Y,Z { \text{ are random variables in } \PP}\}$$
    is decreasing as the precisions of $X$ and $Y$ increase. 
\end{lemma}

\begin{proof}
    We will prove that if $X\_\indep Y|Z$, then $X\indep Y|Z$. 
    
    For all $i\in \Bcal$, for all $(x,y,z)$ in the support of $(X,Y,Z)$, $\PP(xi,y,z)\PP(z) = \PP(xi,z)\PP(y,z)$. Therefore,
    \begin{align}
         \PP(x,y,z)\PP(z) &= \left[ \PP(x0,y,z) + \PP(x1,y,z) \right] \PP(z)\\
         &= \PP(x0,z)\PP(y,z) + \PP(x1,z)\PP(y,z)\\
         & = \PP(x,z)\PP(y,z)
    \end{align}
    which implies $\PP(X,Y|Z)=\PP(X|Z)\PP(Y|Z)$.
\end{proof}

\subsection{Compression in information theory}\label{sec:comment_entropy}
Here we recall some basic results in information theory. We follow \cite{cover1999elements}.

\begin{definition}
    A \emdef{codebook}\footnote{In most literature in information theory~\citep{shannon1948mathematical, cover1999elements} the word ``codebook'' only denotes a code for $\Xcal^N$ with $N$ samples in the noisy-channel coding setting. We abuse the usage of this word in the source coding setting because we would like to stress that the source code itself also needs to be encoded and it also has a coding length, i.e. the coding length of the code(book).} (or source code) $c$ for a random variable X is a mapping $\Xcal\to \Bcal^*=\{0,1\}^*$. Let $c(x)$ denote the codeword corresponding to $x$ and let $l(c(x))$ denote the length of $c(x)$.

    The expected length $L(c)$ of a codebook $c$ for a random variable $X$ with probability mass function $p(x)$ is given by

$$L(c)=\EE[l(c(x))] = \sum_{x\in\mathcal{X}}p(x)l(x).$$
\end{definition}
\begin{example}
Let $X$ be a random variable with the following distribution and codeword assignment:

$$\begin{aligned}
    &\PP(X=1)=\frac12,\quad\text{codeword }c(1)=0\\
    &\PP(X=2)=\frac14,\quad\text{codeword }c(2)=10\\
    &\PP(X=3)=\frac18,\quad\text{codeword }c(3)=110\\
    &\PP(X=4)=\frac18,\quad\text{codeword }c(4)=111.
    \end{aligned}$$

The entropy $H(X)$ of $X$ is 1.75 bits, and the expected length $L(c)=$ $E[l(X)]$ of this code is also 1.75 bits. Here we have a code that has the same average length as the entropy. We note that any sequence of bits can be uniquely decoded into a sequence of symbols of $X.$ For example, the bit string 0110111100110 is decoded as 134213.
\end{example}
\begin{definition}\label{def:extension_codebook}
    The \emdef{extension} $c^*$ of a codebook $c$ is the mapping from finite-length strings of $\mathcal{X}$ to fınite-length strings of $\Bcal^*$, defıned by
    $$c^*(x_1x_2\cdots x_n)=c(x_1)c(x_2)\cdots c(x_n),$$
    where $c(x_1)c(x_2)\cdots c(x_n)$ indicates concatenation of the corresponding codewords.
\end{definition}
\begin{example}
    If $c(x_1)=00$ and $c(x_2)=11$, then $c(x_1 x_2)=0011$.
\end{example}
\begin{definition}
    A code $c$ is called \emdef{uniquely decodable} if its extension $c^*$ is non-singular, namely, for all $x,x'\in \Xcal^*$ sequences of letters in $\Xcal$ such that $x\neq x'$, $c^*(x)\neq c^*(x')$.
\end{definition}
In other words, any encoded string in a uniquely decodable code has only one possible source string producing it.
\begin{definition}
    A codebook is called a \emdef{prefix code} or an instantaneous code if no codeword is a prefix of any other codeword.
\end{definition}
An instantaneous code can be decoded without reference to future codewords since the end of a codeword is immediately recognizable. Hence, for an instantaneous code, the symbol $x_i$ can be decoded as soon as we come to the end of the codeword corresponding to it.

We cannot assign short codewords to all source symbols and still be prefix-free. The set of codeword lengths possible for instantaneous codes is limited by the following inequality.

\begin{lemma}[Kraft inequality]\label{lem:kraft}
    For any instantaneous code (prefix code) over an alphabet of size $D$ ($=2$ in our case), the codeword lengths $l_1, l_2, \ldots, l_m$ must satisfy the inequality
    $$\sum_iD^{-l_i}\leq1.$$
    Conversely, given $a$ set of codeword lengths that satisfy this inequality, there exists an instantaneous code with these word lengths.
\end{lemma}
\begin{corollary}\label{cor:correpondance_proba_code}\cite[p. 96]{grunwald2007minimum}
    Let $\mathcal{X}$ be a finite or countable sample space. Let $\PP$ be a probability distribution over $\mathcal{X}^n$, the set of sequences of length $n.$ Then there exists a prefix code $c$ for $\mathcal{X}^n$ such that for all $x^n\in\mathcal{X}^n,L_c(x^n)=-\log \PP(x^n)$. $c$ is called the code corresponding to $\PP$.
    
    Similarly, let $c^\prime$ be a prefix code for $\mathcal{X}^n.$ Then there exists a defective probability distribution $\PP^\prime$ such that for all $x^n\in\mathcal{X}^n$, $-\log \PP^{\prime}(x^n)=L_c^{\prime}(x^n)$. $\PP^{\prime}$ is called the probability distribution corresponding to $c^\prime$.
\end{corollary}

\begin{theorem}[Shannon's source coding theorem]\cite[Thm. 5.4.1]{cover1999elements}\label{Thm:Shannon_source_coding}
    Let $l_1^*,l_2^*,\ldots,l_m^*$ be optimal codeword lengths for a source distribution p and a D-ary alphabet, and let $L^*$ be the associated expected length of an optimal code $( L^* = \sum p_il_i^* )$. Then 
    $$H_D(X)\leq L^*<H_D(X)+1.$$
\end{theorem}
\cref{Thm:Shannon_source_coding} corresponds to the following communication game: Alice and Bob share the same codebook $c:\Xcal\to \Bcal^*$, and Alice wants to transmit losslessly a sequence $x_1\dots x_n$ iid sampled from $\PP_X$ to Bob. She encodes the sequence into $c^*(x_1\dots x_n)=c(x_1)\dots c(x_n)$ and sends the code sequence to Bob. Bob decodes the sequence using the same codebook.~\cref{Thm:Shannon_source_coding} proves that in this case, asymptotically the minimum expected binary coding length for a word $x\in\Xcal$ sampled from $\PP_X$ is around $H_2(X)$.

Suppose Alice and Bob share nothing except a programming language (C, python), or a universal Turing machine. In that case, the overall bits Alice needs to send can be two parts: a codebook, and a binary codeword sequence. 

\begin{definition}[Huffman coding program~\citep{mackay2003information}]\label{def:Huffman_code}
    A Huffman coding program is a Turing machine that proceeds as follows:
    
    Input: the probability space $\Xcal^d$, and a distribution $\PP_X$ as a discrete function.
    \begin{enumerate}
        \item Take the two least probable symbols in the alphabet $\Xcal^d$. These two symbols will be given the longest codewords, which will have equal length and differ only in the last digit.
        \item Combine these two symbols into a single symbol, and repeat.
    \end{enumerate}
    Output: a tree-structured codebook.
    We call the output of this program a \emdef{Huffman code}.
\end{definition}
\begin{theorem}[Huffman coding is optimal~\citep{cover1999elements}]\label{thm:huffman_optimal}
    If $c^*$ is a Huffman code for $\PP_X$ and $c'$ is any other uniquely decodable code, $\EE_{\PP_X}[l(c^*(X))] \leq \EE_{\PP_X}[l(c'(X))] $.
\end{theorem}
Therefore in the paper, we use the Huffman coding program to transform a multi-env distribution into a codebook. No matter which coding program we choose (e.g. Shannon-Fano-Elias code or arithmetic code), the expected codeword length for each $x\in\Xcal^d$ is close to $\log\PP_X(x)$. Since we can fix a coding program in a UFCC and still compute all finite codebooks, different choices of coding programs do not change significantly the theoretical result of model selection.

%% file: App_UFCC.tex
\section{Proof in~\texorpdfstring{\cref{sec:learn_compression}}{reference}}\label{sec:proof_undecidable}
\FCMundecidable*
\begin{proof}
We call the set above $\Fcal$, and suppose $\Fcal$ is decidable. Then there exists a Turing machine $D$ which inputs the index $k$ of any Turing machine $T_k$, and outputs a decision whether $T_k \in \Fcal$.

We reduce the halting problem to the above decision problem. Given any Turing machine $M$ and input $x$, we construct a new Turing machine $T_{M,x}$:
\begin{algorithmic}
\ttfamily
\State On input $0$, ignore the input $0$ and simulate $M$ on $x$
    \State If $M$ halts on $x$, then delete the output of $M$ and output $0$
    \State If $M$ does not halt on $x$, then rejects
\end{algorithmic}
We run the TM $D$ on $T_{M,x}$. By construction, $T_{M,x}$ either computes a trivial codebook $0\mapsto 0$, or rejects. If $D$ accepts $T_{M,x}$, then $T_{M,x}$ computes $0\mapsto 0$ and halts. If $D$ rejects $T_{M,x}$, then $M$ does not halt on $x$.

Therefore $D$ decides the halting problem, which is a contradiction.
\end{proof}

%% file: App_case_studies.tex
\section{Supplemental results and proofs in~\texorpdfstring{\cref{sec:case_studies}}{reference} Case studies}\label{sec:supp_case_studies}
\subsection{Lemmata}
\begin{lemma}[Bernoulli's Inequality~\citep{carothers2000real}]\label{lem:bernoulli}
    If $a>-1$, $a\neq 0$, then $(1+a)^n > 1+na$ for any integer $n>1$.
\end{lemma}

\begin{lemma}\label{lem:product_precision}
    Given $x\in\QQ \cap [0,1)$ in binary precision $n$, the minimum number of bits $k$ required such that for all exact factorization $x = \prod_{i=1}^l y_i$ with $y_i\in\QQ \cap [0,1)$, 
    $\prod_{i=1}^l y_i|_k$ has the same binary representation as $x$ with truncated precision $n$, i.e.
   \[
   0 \leq x - \prod_{i=1}^l y_i|_k  \leq 2^{-n-1},
   \]
    
    is at most $2n+\log_2 l + 3$.
\end{lemma}
\begin{proof}

Let \( (x_i)_i \)  be the approximations of \( y_i \) with \( k \) bits of precision, i.e. $(y_i|_k)_i$, so:
\begin{equation}
    0 \leq y_i - x_i \leq \delta := 2^{-k}
\end{equation}
Define the relative errors:
\begin{equation}
    \epsilon_i = \frac{y_i - x_i}{y_i}
\end{equation}

Case 1: 
$y_i\geq 2^{-n-2}$ for all $i$. Then we have:
\begin{equation}
    \epsilon_i \leq \frac{\delta}{2^{-n-2}} = 2^{n - k +2} =:\epsilon
\end{equation}
We aim to bound the error:
\begin{equation}
    \prod_{i=1}^l y_i - \prod_{i=1}^l x_i  = x \left( 1 - \prod_{i=1}^l (1 - \epsilon_i) \right) \leq  1 - \prod_{i=1}^l (1 - \epsilon_i)  \leq  1 - (1 - \epsilon)^l \leq l \epsilon
\end{equation}
where the rightmost inequality is by~\cref{lem:bernoulli}.

We want to find $k$ such that the RHS is upper bounded by $2^{-n-1}$:
\begin{equation}
    l \cdot 2^{n - k+2} \leq 2^{-n-1}
\end{equation}
which is equivalent to
\begin{equation}
    k\geq 2n+\log l +3.\\
\end{equation}

Case 2: There exist $i\in[d]$ such that
$y_i< 2^{-n-2}$ for all $i$. Then we have:

\begin{equation}
    0\leq \prod_{j=1}^l (y_j|_k) \leq y_i|_k < 2^{-n-2}
\end{equation}
and
\begin{equation}
    0 \leq x=\prod_{j=1}^l y_j \leq y_i|_k < 2^{-n-2}.
\end{equation}
Therefore, 
\begin{equation}
    0 \leq x - \prod_{i=1}^l y_i|_k  < 2^{-n-1}
\end{equation}
\end{proof}

\subsection{Proofs in~\texorpdfstring{\cref{sec:case_studies}}{reference}}
\propFactorizeShorter*
\begin{proof}
    We construct a CFMP $\alpha$, following~\cref{def:cond_feat_mechanism_program}:

    First step, $\alpha$ generates $\Pcal_{\alpha}$, a list of probability tables:
    \begin{itemize}
        \item for all $i\in [d]$, $f_i$ computes: $ \PP^i(\cdot| \cdot): \Xcal \times (\Xcal^{|S_i|}\times [I])$, i.e. the shifted mechanisms of $X_1$;
        \item $f_{d+1}$ computes the $X_2, \dots, X_d$ marginal of $\PP$;
        \item $f_{d+2}: \text{Unif}[I]$, a uniform prior over environments.
    \end{itemize}
    We want the probabilistic mechanisms output precise enough probability values such that the joint distribution computed at the end of step 3 is lossless. By~\cref{lem:product_precision}, the output precision in each probabilistic mechanism needs $2n+\log 2 + 3 = 2n+4$ bits.
    
    For each $i\in [d]$, $f_i$ is a table with input space at most $(\Bcal^m)^{1+(d-2)} \times \Bcal^{ \log I
    }$, and output space $\Bcal^{2n+4}$. The coding length of $f_i$ as a table is at most $(2n+4) (2^m)^{d-1} I $. The coding length of $f_{d+1}$ is $(2n+4) (2^m)^{d-1}$. The coding length of $f_{d+2}$ is $(2n+4) I $, since $\frac{1}{2^n}\leq \frac{1}{I} = \frac
    {1}{2^{\log I}}$ by assumption. In total we need $(2n+4) ( (2^m)^{d-1} I + (2^m)^{d-1} + I)$ bits for $\Pcal_\alpha$.

    By assumption in UFCC, $\Phi_\alpha$ only contains projections. Since the cardinal of the power set of the nodes is $2^{d+1}$, we need $d+1$ bits for each feature mechanism. In total we need $(I+3)(d+1)$ bits for $\Phi_\alpha$.

    Second step, we assign $I+3$ feature mechanisms to $I+2$ probabilistic mechanisms, each of which has two cases (conditional or value) for filling the feature mechanism. So the whole mapping costs $2(I+2)\log (2I+3)$ bits.

    Third step, for each $(x,e)\in \Xcal^{d} \times [I]$, $\alpha$ computes $\PP(x, e) = f_e(x_1|x_{S_e}) f_{d+1}(X_2, \dots, X_d)f_{d+2}(e)$. We need $O(I\log I)$ bits because we only need to assign $I$-many mechanisms $(f_i)_{i\in [I]}$ to $I$-many environments, with other mechanisms being the same for all $(x,e)\in \Xcal^{d} \times [I]$.

    The Huffman coding mechanism~\cref{def:huffman_mechanism} that turns a distribution into a Huffman code is of constant coding length. In total, the coding length of $\alpha$ is $O(nI (2^m)^{d-1} + Id)$.

    Define $\beta$ as the CFMP that encodes the whole multi-env distribution: $\Pcal_\beta$ has only one element: a probability table $\PP$ of precision $n$. This costs $nI2^{m^d}$ bits. $\Phi_\beta$ only contains the identity, which costs $d+1$ bits, same as those elements in $\Phi_\alpha$. Since the joint distribution is already represented, all the following steps are trivial and only cost constant bits. Therefore, $l_{U_{\text{tabCBN}}} (\alpha)(m,d) = o(l_{U_{\text{tabCBN}}} (\beta)(m,d))$.

\end{proof}

\propSMS*
\begin{proof}
    There are $\binom{M}{k}$ possibilities for choosing $k$ mechanisms from $\Pcal_\alpha$. There are $k! S_k^N$ possibilities for subjective functions $[N]\to [k]$, where $S_k^N$ is the Stirling number of the second kind. Given $k$, the bits for step 2 in strategy 2 is $\log\binom{M}{k} + \log (k!) + \log k$. The bits for step 2 in strategy 1 are $N\log M$. So the difference between strategy 2 and strategy 1 is
    \begin{equation}\label{eq:diff_bits}
        A(k)=\log\binom{M}{k} + \log S_k^N + \log (k!) + \log k -N\log M
    \end{equation}
    So the difference between $A(k+1)$ and $A(k)$ is
    \begin{align}
        A(k+1) - A(k) &= \log \left( \frac{M - k}{k + 1} \right) + \log \left( \frac{S_{k+1}^N}{S_k^N} \right) + \log (k+1) + \log \left( \frac{k+1}{k} \right)\\
        &= \log \left( (M - k) \cdot \frac{S_{k+1}^N}{S_k^N} \cdot \frac{k+1}{k} \right).\label{eq:SMS_Ak+1-Ak}
    \end{align}
    For $k=o(N)$, we have an approximation for Stirling number of the second kind: $\lim_{N\to \infty} S_k^N= \frac{k^N}{k!}$. Then
    \begin{align}
        \lim_{N\to \infty}\frac{S_{k+1}^N}{S_k^N} &=\frac{\frac{(k+1)^N}{(k+1)!}}{\frac{k^N}{k!}} = \frac{(k+1)^N}{k^N} \cdot \frac{k!}{(k+1)!} = \left(1 + \frac{1}{k}\right)^N \cdot \frac{1}{k + 1}\label{eq:SMS_frac_SN}
    \end{align}
In order that \cref{eq:SMS_Ak+1-Ak} is greater than zero, by \cref{eq:SMS_frac_SN} we need for any $k= o(N)$,
    \begin{equation}
        \frac{M-k}{k} \left(1+\frac{1}{k}\right)^N>1
    \end{equation}
    \begin{equation}\label{eq:SMS_k<k0}
        M>\frac{k}{\left(1+\frac{1}{k}\right)^N}+k,
    \end{equation}

    which holds when $k<\frac{M}{2}$.
    
    Therefore, $A(k)$ monotonically increases when $k=o(N)$ and $k<\frac{M}{2}$. By~\cref{eq:diff_bits}, 
    \begin{equation}
        A(1) = (1-N)\log M + \log 2 +1<0
    \end{equation}

    when $N$ is sufficiently large. The conclusion of the proposition is obtained from $-A$.
\end{proof}

\propInvShorter*
\begin{proof}
    The proof is trivial because the steps in $\alpha$ and $\beta$ are the same only except for $\Pcal_\alpha$ and $\Pcal_\beta$. 
    
    $\Pcal_\alpha$ consists of two tables: $f_1(\cdot| \cdot): \Xcal \times \phi(\Xcal) \to \Bcal^n$, $f_{2}:\Xcal\to \Bcal^n$. By~\cref{lem:product_precision}, each probability value in each probabilistic mechanism table needs $2n+\log 2 + 1 = 2n+4$ bits. Therefore, $f_1$ needs $(2n+4)2^{m+c}$ bits, $f_2$ needs $(2n+4)2^{m}$ bits, where $c$ is a constant independent of $m$.

    $\Pcal_\beta$ consists of two tables: $f'_1(\cdot| \cdot): \Xcal \times \Xcal \to \Bcal^n$, $f_{2}:\Xcal\to \Bcal^n$. $f_1$ needs $(2n+4)2^{m^2}$ bits, $f_2$ needs $(2n+4)2^{m}$ bits.

    Therefore, $l_{U_{\text{TabInv}}} (\alpha(m)) = o(l_{U_{\text{TabInv}}} (\beta(m)))$.

\end{proof}

%% file: App_experiments.tex
\section{Details of experiments in \texorpdfstring{\cref{sec:experiments}}{reference}}\label{sec:details_experiments}
\subsection{Covariate shifts}

The parameters of Poisson distribution $\PP(X|E)$ for each env are:

Blue curve ($k=7$):  $[0.1,0.1,0.1,0.3,0.3 ,0.5,0.6,0.7,0.8,0.9]\times 20$.

Green curve ($k=4$): $[0.1,0.1,0.4,0.4,0.6, 0.6,0.6,0.9,0.9,0.9]\times 20$.
 
Red curve ($k=2$): $[0.1,0.1,0.1,0.1,0.1, 0.1,0.1,0.1,0.9,0.9]\times 20$.

We set the list of Poisson parameters for any env to choose from:

$[0.05, 0.1,0.15, 0.2, 0.25, 0.3, 0.35,0.4, 0.45, 0.5, 0.55, 0.6, 0.65,0.7, 0.75, 0.8, 0.85, 0.9]\times 20$.

For each $k\in [10]$, we use greedy search over all possible choices of $k$-many parameters, i.e. $\binom{18}{k}$. For each choice, we assign those $k$ parameters to the environment that has the closest estimated mean, which is the estimated $\lambda$ for Poisson distribution.

\subsection{Causal discovery without identifiability}

The 5-env data is generated by the following Gaussian parameters $(\sigma_1^2,\sigma_2^2,a)$:
\begin{align*}
    e=0:& [1, 16, 1],\\
    e=1:& [1, 16, 1],\\
    e=2:& [9, 16, 2],\\
    e=3:& [9, 25, 2.5],\\
    e=4:& [25, 25, 2.5].
\end{align*}

For the causal model $Y=a X + \epsilon$, we select the parameters from a uniformly distributed set of $8$ points spanning the range $[p_{\text{min}}, p_{\text{max}}]$ for each parameter $p\in\{(\sigma_1^2,\sigma_2^2,\gamma)\}$. For the anti-causal model $X=b Y+\epsilon$, we compute the corresponding parameters in each env that achieve the same likelihood as the ground truth causal model, and create the set of candidate parameters in the same way as the causal model.

%% file: App_AMC.tex
\section{Remarks on Algorithmic Markov Condition (AMC)}\label{sec:remark_AMC}
Algorithmic Markov Condition (AMC) was introduced in \cite{janzing2010causal}. We first briefly recall their paper's main result, then compare the principles in it and in our paper. In the following, they use $K$ to denote \emdef{prefix Kolmogorov complexity} \citep[Chapter 3]{li2019introduction}.
\begin{definition}[algorithmic Markov condition]\citep[Postulate 5]{janzing2010causal}\label{def:AMC}
Let $x_1,\ldots,x_n$ be $n$ strings representing descriptions of observations whose causal connections are formalized by $a$ directed acyclic graph $G$ with $x_1,\ldots,x_n$ as nodes. Let pa$_j$ be the concatenation of all parents of $x_j$ and $nd_j$ the concatenation of all its non-descendants except $x_j$ itself. Then
$$x_j\perp nd_j\mid pa_j^*.$$
where $pa_j^*=(pa_j, K(pa_j))$, and the above independence is algorithmic, i.e. the algorithmic conditional mutual information $I(x_j:nd_j|pa_j^*):=K(x_j|pa_j^*)+K(nd_j|pa_j^*)-K(x_j:nd_j|pa_j^*)+O(1) =O(1)$.
\end{definition}

Then they prove in their Thm. 3 that the above condition is equivalent to the statement 
\begin{equation}\label{eq:AMC_string}
    K(x_1,\dots, x_n)=\sum_{j=1}^n K(x_j|pa_j^*) +O(1).
\end{equation}
In other words, for $n$ strings $x_1, \dots, x_n$, if \cref{eq:AMC_string} does not hold for a DAG $G$ then we reject the statement that $x_1, \dots, x_n$ satisfy an algorithmic causal model with DAG $G$, see their Postulate 6.

Inspired by \cref{eq:AMC_string}, they propose a principle for the DAG selection given a joint distribution $P$:
\begin{principle}\citep[Postulate 7]{janzing2010causal}\label{pr:AMC_postulate7}
A causal hypothesis $G$ (i.e., a DAG) is only acceptable if the shortest description of the joint
density $P$ is given by a concatenation of the shortest description of the Markov kernels, i.e. 
\begin{equation}\label{eq:AMC_postulate7}
    K(P(X_1,\ldots,X_n))=\sum_jK(P(X_j|PA_j))+O(1)
\end{equation}
where $K(P)$ is the length of the shortest prefix-free program that computes $P(x, y)$ from $(x, y)$.
If no such causal graph exists, we reject every possible DAG and assume that there is a causal relation of a different type, e. g., a latent common cause, selection bias, or a cyclic causal structure.
\end{principle}
We note that although \cref{eq:AMC_postulate7} seems similar to \cref{eq:AMC_string}, they are in fact different principles and one cannot derive one from the other. By \cref{def:KC}, the Kolmogorov complexity, whether prefix ($K_U$) or not ($C_U$), is a function that inputs strings instead of functions. A function has multiple string representations, so does a function or distribution $P$ over $(X_1,\ldots,X_n)$. If we consider $P$ as $n$ strings and apply \cref{eq:AMC_string} to them, this is ill-defined because of the non-uniqueness of string representations. In fact, the motivation of proposing \cref{pr:AMC_postulate7} as a model selection principle is not related to \cref{eq:AMC_string}, but to what is afterwards named as ``independent causal mechanism'' (ICM) principle \citep{peters2017elements}:

``We can think of $P(X)$ as describing a source $S$ that generates $x$-values and sends them to a ``machine'' $M$ that generates $y$-values according to $P (Y |X)$. Assume we observe that $I(P(X): P(Y|X)) \gg 0$.\footnote{$I$ denotes the algorithmic mutual information, see \cref{def:AMC}.} Then we conclude that there must be a causal link between $S$ and $M$ that goes beyond transferring $x$-values from $S$ to $M$. This is because $P (X)$ and $P (Y |X)$ are inherent properties of $S$ and $M$ , respectively which do not depend on the current value of $x$ that has been sent.'' \citep[Sec. 3.1]{janzing2010causal}

In other words, if the shortest Turing machines $S$ and $M$ computing respectively $P(X)$ and $P(Y|X)$ have algorithmic mutual information $O(1)$, then they accept the DAG $X\to Y$, otherwise reject.

Our principle of model selection, \cref{priniple:compression}, does not imply \cref{pr:AMC_postulate7}. The reasons and comparisons are the following:
\begin{itemize}
    \item \cref{priniple:compression} allows selecting a CFMP that contains probabilistic mechanisms that are not algorithmically independent, for example $P(X)\sim \Ncal(0,\sigma^2)$ and $P(Y|X)\sim \Ncal(X,\sigma^2)$, which, according to \cref{pr:AMC_postulate7}, leads to rejecting $X\to Y$ because of the compressibility of the shared parameter $\sigma$.
    \item Our output of model selection is different from \cref{pr:AMC_postulate7}. We select a Turing machine, from which the causal and symmetry statements are read off. \cref{pr:AMC_postulate7} selects a graph only among all possible DAGs, which is less than the possible models that we illustrate in \cref{eg:CFMP}.
    \item The communication game setting behind \cref{pr:AMC_postulate7} is: Alice and Bob share a universal Turing machine, and Alice would like to send a string as short as possible to Bob so that Bob could compute a function $P$. In constrast, we are interested in the game where Alice would like to send \textit{datasets}. The idea of UFCC is that Alice should send a codebook and a codeword so that Bob could reconstruct the datasets. A codebook might consists of a probability distribution function $P$ or not. Our choice of $P$ depends on the trade-off between the two part codes, instead of being given a priori in \cref{eq:AMC_postulate7}.
\end{itemize}
The last point shows that the two-part code objective in the sense of MDL or UFCC is fundamentally different from \cref{pr:AMC_postulate7}. \cite{marx2021formally} aim at linking \cref{pr:AMC_postulate7} and two-part code in MDL. They propose a two-part code objective adapted from \cite{Budhathoki}:
\begin{equation}\label{eq:AMC_MDL1}
    K_{X\to Y}:= K(P_X)+ K(x|P_X) + K(P_{Y|X}) + K(y|x, P_{Y|X})
\end{equation}
and they proved that $K_{X\to Y}$ is on expectation equal to 
\begin{equation}\label{eq:AMC_MDL2}
    K(P_X) + K(P_{Y|X}) + H(P_{XY}),
\end{equation}
which is also a two-part code objective consisting of model length $K(P_X) + K(P_{Y|X})$ and the codeword length $H(P_{XY})$. Their two objectives \cref{eq:AMC_MDL1} and \cref{eq:AMC_MDL2}, however, do not agree with $K(P_{XY})$ in \cref{pr:AMC_postulate7}. Therefore, the two-part code objective (whether from MDL or UFCC) and \cref{pr:AMC_postulate7} cannot be derived from each other.

%% file: App_MDL.tex
\section{Remarks on Minimum Description Length principle (MDL) and Bayesian model selection}\label{sec:MDL}

In \cref{def:UFCC} we define that each codebook computed by a UFCC must be a function $\Xcal^d \to \Bcal^*$. When the data is iid sampled from a distribution over $\Xcal^d$ then Huffman code is optimal (\cref{thm:huffman_optimal}). If we modify the definition of UFCC by changing the domain of codebook from $\Xcal^d$ to $(\Xcal^d)^*$, and if the data is exchangeable instead of iid, then there is a code called Bayes code that has shorter codeword length than a given Huffman code for all sequence $x\in (\Xcal^d)^*$. This is a fundamental idea in MDL principle  \citep{grunwald2007minimum}:

\begin{example}[Example 6.4 in \cite{grunwald2007minimum}: Bayes code is better than two-part code]\label{eg:Bayes_code_better}

The Bayesian model is in a sense superior to the two-part code. Namely, in the two-part code we first encode an element in the parameter set $\Theta$ using some code $L_0.$ Such a code must correspond to some "prior" distribution $W$ on $\mathcal{M}$ so that the two-part code gives codelengths
\begin{equation}\label{eq:two_part_MDL}
    L_\text{2-part}(x^n)=\min_{\theta\in\Theta}-\log \PP_\theta(x^n)-\log W(\theta),
\end{equation}

where $W$ depends on the specific code $L_0$ that was used. 

Define the Bayes code with prior $W$:
\begin{equation}\label{eq:Bayes_code_MDL}
    L_{\text{Bayes}}(x^n):=-\log \PP_\text{Bayes}(x^n)=-\log\sum_{\theta\in\Theta}\PP_\theta(x^n)W(\theta)
\end{equation}

where $P_\text{Bayes}$ is the marginal likelihood of the data $x^n$ under the prior $W$. Then it is direct to see that

$$L_{\text{Bayes}}(x^n)=-\log\sum_{\theta\in\Theta}\PP_\theta(x^n)W(\theta)\leq\min_{\theta\in\Theta}-\log \PP_\theta(x^n)-\log W(\theta)=L_\text{2-part}(x^n)$$
because a sum is at least as large as each of its terms.

The inequality becomes strict whenever $\PP_\theta(x^n)>0$ for more than one value of $\theta$. We see that in general the Bayesian code is preferable over the two-part code: for all $x^n$ it never assigns code lengths larger than $L_{2-\text{part}}( x^n)$, and in many cases it assigns strictly shorter codelengths for some $x^n$.    
\end{example}

The above example shows that for any two-part code there exists a Bayes code that is uniformly shorter than that two-part code. Therefore, the MDL research prefers Bayes code to two-part code, respectively termed refined MDL and crude MDL in \cite{grunwald2007minimum}. 

Using this example, we can define a \emdef{Universal Bayes Codebook Computer (UBCC)}\footnote{Same as \cref{def:UFCC} we omit $(m,d)$ in the input for simplicity. In the general case we can construct UBCC depending on $(m,d)$ by inputting $\langle m, \langle d, \langle p \rangle \rangle \rangle$ and simulating the codebook for each $(m,d)$ respectively.}:
\begin{definition}\label{def:UBCC}
    A Turing machine is called Universal Bayes Codebook Computer (UBCC) if it is constructed as follows:
\begin{enumerate}
    \item First, same as in \cref{def:UFCC}, define any recursively enumerable set $S$ of FCMs, such that any codebook $g:\Xcal^d \to \Bcal^*$ can be computed by at least one of the FCMs in it (In the following, we will call such a r.e. set a universal set of FCMs.).
    \item Define a discrete probability $\QQ$ fully supported over that countable set $\Scal$.
    \item For any $x\in (\Xcal^d)^*$, compute its marginal likelihood $\PP(x)= \sum_{T\in \Scal} \PP(x|T)\QQ(T)$, where $\PP(\cdot|T)$ is the probability distribution function computed by $T$. By \cref{cor:correpondance_proba_code}, the negative log marginal likelihood $-\log \PP$ is the coding length of a certain Shannon code over $(\Xcal^d)^*$. By \cref{eg:Bayes_code_better}, this code is shorter than any two-part code in any UFCC using the same r.e. set $\Scal$ of FCMs. 
\end{enumerate}
And different from FC complexity which is a two-part code, we define \emdef{Bayes codebook complexity (BC)} $C^{\text{BC}}_V(\cdot)$ as a one-part code, i.e. the shortest integer $p\in \NN$ such that the codebook corresponding to $-\log \PP$ above can decode $B(p)$ and output $x$.
\end{definition}
In other words, a UBCC is one single infinite codebook mechanism, over $(\Xcal^d)^*$. Given a UFCC $V$, we obtain a r.e. set $\Scal$ of FCMs, and we can build a UBCC $V'$ using $\Scal$ by assigning each FCM $T\in \Scal$ to a prior probability $\QQ(T)= 2^{-l_V(T)}$. By \cref{eg:Bayes_code_better}, $C^{\text{BC}}_{V'}(x)\leq C^{\text{FC}}_V(x)$ for all $x\in (\Xcal^d)^*$. Namely, $C^{\text{BC}_{V'}}$ is a more refined upper bound of Kolmogorov complexity than $C^{\text{FC}_V}$, although the r.e. sets of Turing machines that $V$ and $V'$ simulate are disjoint: $V'$ always \textit{computes} the same codebook and feed it with different binary inputs to output different $x$, while $V$ \textit{simulates} explicitly each FCM of which the input is a binary sequence.

However, to proceed the model selection over $\Scal$, Bayesian model selection in UBCC and two-part code model selection in UFCC coincide: they are both maximum a posteriori. For UFCC, the prefered FCM is $\min_{T,p}\{2l_V(T) + l(p)\}$ which equals $-\log\QQ(T) -\log\PP(x|T)$ for certain $\PP$ and $\QQ$ (the existence of them are again by \cref{cor:correpondance_proba_code}). For UBCC, Bayesian model selection chooses the $T$ that maximizes the posterior $\PP(T|x)$.

The Bayesian model selection in \cite{dhirbivariate} can be considered as a compression scheme between UFCC (pure two-part code) and UBCC (pure Bayes code). They defined their decision criterion of causal graph as the ratio of posterior

\begin{equation}
    \log\frac{\PP(G_1|x)}{\PP(G_2|x)}=\log\frac{\PP(x|G_1)\PP(G_1)}{\PP(x|G_2)\PP(G_2)}
\end{equation}

and $G_1$ is preferred to $G_2$ if the log ratio is positive. To represent the lack of knowledge over graph choices, they set the prior over graphs to be uniform. Since they choose the graph $G$ that maximizes $\PP(x|G)\PP(G)$, which is equivalent to minimize $-\log \PP(x|G) - \log \PP(G)$, their objective is also a two-part code: first encode the graph, and then encode a \textit{Bayes code} (negative log marginal likelihood) of $x$ given $G$.

The main difference between \cite{dhirbivariate} and the computational-theoretic objective (with reference machine UFCC or UBCC) is that the former approach \citep{dhirbivariate} aims at maximizing the posterior of a \textit{graph}, while the latter aims at maximizing the posterior of a \textit{Turing machine}, which, from \cite{dhirbivariate} or a probabilistic point of view, determines a graph and conditional probabilities on it. For \cite{dhirbivariate}, the conditional probabilities in the selected model are uncertain under a given prior. The preference of theories on causality is, in our view, fundamentally subjective.

Here we summarize all the compression games we mentioned in our paper:
\begin{itemize}
    \item Shannon source coding: as explained after~\cref{Thm:Shannon_source_coding}, Alice and Bob know the data distribution $\PP_X$. Alice wants to send iid samples losslessly to Bob using a codebook. Before sending iid data, they can design together a codebook.\\ Question: what is the shortest expected average coding length for each sample, among all possible codebooks? (without counting the length of the codebook)
    \item Kolmogorov complexity: Alice and Bob share a programming language (C, Python) or a universal Turing machine. They do not know any structure of the data sequence $x$ to be sent.\\ Question: what is the length of the shortest program that Alice can send so that Bob can losslessly recover $x$?
    \item Finite codebook (FC) complexity (\cref{def:FC_complexity}): Alice and Bob share a finite universal codebook computer (UFCC,~\cref{def:UFCC}). They know that the data sequence $x$ is sampled in $\Xcal^d$ with precision $m$. Before sending data, they can design together a codebook $\Xcal^d \to \Bcal^*$.\\ Question: what is the length of the shortest program that Alice can send so that Bob can losslessly recover $x$? Since we define the UFCC to only accept two-part code (codeword and codebook), the question is equivalent to: what is the minimal sum of the length of the codewords and the length of the codebook mechanism (TM) to compress a certain sequence $x$?
    \item Bayes codebook (BC) complexity (\cref{def:UBCC}): Alice and Bob share a universal Bayes codebook computer (UBCC,~\cref{def:UBCC}). They know that the data sequence $x$ is sampled in $\Xcal^d$ with precision $m$. Before sending data, they can design together a codebook $(\Xcal^d)^* \to \Bcal^*$.\\ Question: what is the length of the shortest binary string (Bayes codeword) that Alice can send so that Bob can losslessly recover $x$?
\end{itemize}